\documentclass{article}
\usepackage{matus}

\author{Ziwei Ji\qquad Matus Telgarsky\\
\tt{\{\href{mailto:ziweiji2@illinois.edu}{ziweiji2},\href{mailto:mjt@illinois.edu}{mjt}\}@illinois.edu}\\
University of Illinois, Urbana-Champaign}
\date{}

\newcommand\polylog{\ensuremath{\operatorname{polylog}}}

\def\w#1#2{w_{#1,#2}}
\def\W#1{W_{#1}}
\def\setdef#1#2{\cbr{#1\ {} \middle| \ {} #2}}
\def\barU{\overline{U}}
\def\barW{\overline{W}}
\def\hatv{\hat{v}}
\def\hatq{\hat{q}}
\def\wtOmega{\widetilde{\Omega}}
\def\wtO{\widetilde{O}}
\def\wtTheta{\widetilde{\Theta}}

\title{Polylogarithmic width suffices for gradient descent to achieve
  arbitrarily small test error with shallow ReLU networks}

\begin{document}

\maketitle

\begin{abstract}
  Recent theoretical work has guaranteed that overparameterized networks trained
  by gradient descent achieve arbitrarily low training error, and sometimes even
  low test error.
  The required width, however, is always polynomial in at least one of the
  sample size $n$, the (inverse) target error $\nicefrac 1 \eps$, and the
  (inverse) failure probability $\nicefrac 1 \delta$.
  This work shows that $\wtTheta(\nicefrac 1 \eps)$ iterations of gradient
  descent with $\wtOmega(\nicefrac 1 {\eps^2})$ training examples on two-layer
  ReLU networks of any width exceeding
  $\polylog(n, \nicefrac 1 \eps, \nicefrac 1 \delta)$ suffice to achieve a test
  misclassification error of $\eps$.
  We also prove that stochastic gradient descent can achieve $\eps$ test error
  with polylogarithmic width and $\wtTheta({\nicefrac 1 \eps})$ samples.
  The analysis relies upon the separation margin of the limiting kernel, which
  is guaranteed positive, can distinguish between true labels and random
  labels, and can give a tight sample-complexity analysis in the infinite-width
  setting.
\end{abstract}

\section{Introduction}\label{sec:intro}

Despite the extensive empirical success of deep networks,
their optimization and generalization properties are still not
fully understood.
Recently, the neural tangent kernel (NTK) has provided the following insight into the
problem.
In the infinite-width limit, the NTK converges to a limiting kernel which stays
constant during training; on the other hand, when the width is large enough, the
function learned by gradient descent follows the NTK \citep{jacot_ntk}.
This motivates the study of overparameterized networks trained by gradient
descent, using properties of the NTK.
In fact, parameters related to the NTK, such as the minimum eigenvalue of the
limiting kernel, appear to affect optimization and generalization
\citep{arora_2_gen}.

However, in addition to such NTK-dependent parameters, prior work also requires
the width to depend polynomially on $n$, $1/\delta$ or $1/\epsilon$, where $n$
denotes the size of the training set, $\delta$ denotes the failure probability,
and $\epsilon$ denotes the target error.
These large widths far exceed what is used empirically,
constituting a significant gap between theory and practice.

\paragraph{Our contributions.}

In this paper, we narrow this gap by showing that a two-layer ReLU network
with $\Omega(\ln(n/\delta)+\ln(1/\epsilon)^2)$ hidden units trained by gradient
descent achieves classification error $\epsilon$ \emph{on test data}, meaning
both optimization and generalization occur.
Unlike prior work, the width is fully polylogarithmic in $n$, $1/\delta$,
and $1/\epsilon$; the width will additionally depend on the
\emph{separation margin} of the limiting kernel, a quantity which is guaranteed
positive (assuming no inputs are parallel), can distinguish between true labels
and random labels, and can give a tight sample-complexity analysis in the
infinite-width setting.
The paper organization together with some details are described below.
\begin{description}
    \item[\Cref{sec:erm}] studies gradient descent on the training set.
    Using the $\ell_1$ geometry inherent in classification tasks, we prove that
    with any width at least polylogarithmic and any constant step size no larger
    than $1$, gradient descent achieves training error $\eps$ in
    $\wtTheta(1/\epsilon)$ iterations (cf.  \Cref{fact:erm}).
    As is common in the NTK literature \citep{bach_chizat_note}, we also show
    the parameters hardly change, which will be essential to our generalization
    analysis.

    \item[\Cref{sec:gen}] gives a test error bound.
    Concretely, using the preceding gradient descent analysis, and standard
    Rademacher tools and exploiting how little the weights moved, we show that
    with $\wtOmega(1/\epsilon^2)$ samples and $\wtTheta(1/\epsilon)$ iterations,
    gradient descent finds a solution with $\epsilon$ test error (cf.
    \Cref{fact:gen} and \Cref{fact:gen_cor}).
    (As discussed in \Cref{fact:gen_smooth}, $\wtOmega(1/\epsilon)$ samples also
    suffice via a smoothness-based generalization bound, at the expense of large
    constant factors.)

    \item[\Cref{sec:sgd}] considers stochastic gradient descent (SGD) with
    access to a standard stochastic online oracle.
    We prove that with width at least polylogarithmic and $\wtTheta(1/\epsilon)$
    samples, SGD achieves an arbitrarily small test error (cf. \Cref{fact:sgd}).

    \item[\Cref{sec:sep}] discusses the separation margin, which is in general a
    positive number, but reflects the difficulty of the classification problem
    in the infinite-width limit.
    While this margin can degrade all the way down to $O(1/\sqrt{n})$ for random
    labels, it can be much larger when there is a strong relationship between
    features and labels: for example, on the \emph{noisy 2-XOR} data introduced
    in \citep{wei_reg}, we show that the margin is $\Omega(1/\ln(n))$, and our
    SGD sample complexity is tight in the infinite-width case.

    \item[\Cref{sec:open}] concludes with some open problems.
\end{description}

\subsection{Related work}\label{sec:rw}

There has been a large literature studying gradient descent on overparameterized
networks via the NTK.
The most closely related work is \citep{nitanda_refined},
which shows that a two-layer network trained by gradient descent with the logistic
loss can achieve a small test error, under the same assumption that the NTK with
respect to the first layer can separate the data distribution.
However, they analyze smooth activations, while we handle the ReLU.
They require $\Omega(1/\epsilon^2)$ hidden units, $\wtOmega(1/\epsilon^4)$
data samples, and $O(1/\epsilon^2)$ steps, while our result only needs
polylogarithmic hidden units, $\wtOmega(1/\epsilon^2)$ data samples, and
$\wtO(1/\epsilon)$ steps.

Additionally on shallow networks, \citet{du_2_opt} prove that on an
overparameterized two-layer network, gradient descent can globally minimize the
empirical risk with the squared loss.
Their result requires $\Omega(n^6/\delta^3)$ hidden units.
\citet{oymak_moderate,song_quad} further reduce the required
overparameterization, but there is still a $\poly(n)$ dependency.
Using the same amount of overparameterization as \citep{du_2_opt},
\citet{arora_2_gen} further show that the two-layer network learned by gradient
descent can achieve a small test error, assuming that on the data distribution
the smallest eigenvalue of the limiting kernel is at least some positive
constant.
They also give a fine-grained characterization of the predictions made by
gradient descent iterates; such a characterization makes use of a special
property of the squared loss and cannot be applied to the logistic regression
setting.
\citet{li_liang_nips} show that stochastic gradient descent (SGD) with the cross
entropy loss can learn a two-layer network with small test error, using
$\poly(\ell,1/\epsilon)$ hidden units, where $\ell$ is at least the covering
number of the support of the feature distribution using balls whose radii
are no larger than the smallest distance between two data points with different
labels.
\citet{allen_3_gen} consider SGD on a two-layer network, and a variant of SGD on
a three-layer network.
The three-layer analysis further exhibits some properties not captured by the
NTK.
They assume a ground truth network with infinite-order smooth activations, and
they require the width to depend polynomially on $1/\epsilon$ and some constants
related to the smoothness of the activations of the ground truth network.

On deep networks,
a variety of works have established low training error
\citep{allen_deep_opt,du_deep_opt,zou_deep_opt_1,zou_deep_opt_2}.
\citet{allen_rnn_opt} show that SGD can minimize the regression loss for
recurrent neural networks, and \citet{allen_rnn_gen} further prove a low
generalization error.
\citet{allen_resnet} show that using the same number of training examples, a
three-layer ResNet can learn a function class with a much lower test error than
any kernel method.
\citet{cao_deep_gen_1} assume that the NTK with respect to the second layer of a
two-layer network can separate the data distribution, and prove that gradient
descent on a deep network can achieve $\epsilon$ test error with
$\Omega(1/\epsilon^4)$ samples and $\Omega(1/\epsilon^{14})$ hidden units.
\citet{cao_deep_gen_2} consider SGD with an online oracle and give a general
result.
Under the same assumption as in \citep{cao_deep_gen_1}, their result requires
$\Omega(1/\epsilon^{14})$ hidden units and sample complexity
$\wtO(1/\epsilon^2)$.
By contrast, with the same online oracle, our result only needs polylogarithmic
hidden units and sample complexity $\wtO(1/\epsilon)$.

\subsection{Notation}\label{sec:notation}

The dataset is denoted by $\{(x_i,y_i)\}_{i=1}^{n}$ where $x_i\in\R^d$ and
$y_i\in\cbr{-1,+1}$.
For simplicity, we assume that $\|x_i\|_2=1$ for any $1\le i\le n$, which is
standard in the NTK literature.

The two-layer network has weight matrices $W\in\R^{m\times d}$ and $a\in\R^m$.
We use the following parameterization, which is also used in
\citep{du_2_opt,arora_2_gen}:
\begin{align*}
    f(x;W,a):=\frac{1}{\sqrt{m}}\sum_{s=1}^{m}a_s\sigma\del{\ip{w_s}{x}},
\end{align*}
with initialization
\begin{align*}
    \w{s}{0}\sim \mathcal{N}(0,I_d),\quad\textrm{and}\quad a_s\sim \mathrm{unif}\del{\cbr{-1,+1}}.
\end{align*}
Note that in this paper, $\w{s}{t}$ denotes the $s$-th row of $W$ at step $t$.
We fix $a$ and only train $W$, as in
\citep{li_liang_nips,du_2_opt,arora_2_gen,nitanda_refined}.
We consider the ReLU activation $\sigma(z):=\max\cbr{0,z}$, though our analysis
can be extended easily to Lipschitz continuous, positively homogeneous
activations such as leaky ReLU.

We use the logistic (binary cross entropy) loss $\ell(z):=\ln\del{1+\exp(-z)}$
and gradient descent.
For any $1\le i\le n$ and any $W$, let $f_i(W):=f(x_i;W,a)$.
The empirical risk and its gradient are given by
\begin{align*}
    \hcR(W):=\frac{1}{n}\sum_{i=1}^{n}\ell\del{y_if_i(W)},\quad\textrm{and}\quad\hnR(W)=\frac{1}{n}\sum_{i=1}^{n}\ell'\del{y_if_i(W)}y_i\nabla f_i(W).
\end{align*}
For any $t\ge0$, the gradient descent step is given by
$\W{t+1}:=\W{t}-\eta_t\hnR(\W{t})$.
Also define
\begin{align*}
    f_i^{(t)}(W):=\ip{\nabla f_i(\W{t})}{W},\quad\textrm{and}\quad\hcR^{(t)}(W):=\frac{1}{n}\sum_{i=1}^{n}\ell\del{y_if_i^{(t)}(W)}.
\end{align*}
Note that $f_i^{(t)}(\W{t})=f_i(\W{t})$.
This property generally holds due to homogeneity: for any $W$ and
any $1\le s\le m$,
\begin{align*}
    \frac{\partial f_i}{\partial w_s}=\frac{1}{\sqrt{m}}a_s\1\sbr{\ip{w_s}{x_i}>0}x_i,\quad\textrm{and}\quad\ip{\frac{\partial f_i}{\partial w_s}}{w_s}=\frac{1}{\sqrt{m}}a_s\sigma\del{\ip{w_s}{x_i}},
\end{align*}
and thus $\ip{\nabla f_i(W)}{W}=f_i(W)$.

\section{Empirical risk minimization}\label{sec:erm}

In this section, we consider a fixed training set and empirical risk
minimization.
We first state our assumption on the separability of the NTK, and then give our
main result and a proof sketch.

The key idea of the NTK is to do the first-order Taylor approximation:
\begin{align*}
    f(x;W,a)\approx f(x;\W{0},a)+\ip{\nabla_Wf(x;\W{0},a)}{W-\W{0}}.
\end{align*}
In other words, we want to do learning using the features given by
$\nabla f_i(\W{0})\in\R^{m\times d}$.
A natural assumption is that there exists $\barU\in\R^{m\times d}$ which can
separate $\cbr[1]{\del{\nabla f_i(\W{0}),y_i}}_{i=1}^n$ with a positive margin:
\begin{align}\label{eq:finite_sep}
    \min_{1\le i\le n}\del{y_i\ip{\barU}{\nabla f_i(\W{0})}}=\min_{1\le i\le n}\del{y_i \frac{1}{\sqrt{m}}\sum_{s=1}^{m}a_s \langle\baru_s,x_i\rangle\1\sbr{\langle \w{s}{0},x_i\rangle>0}}>0.
\end{align}

The infinite-width limit of \cref{eq:finite_sep} is formalized as
\Cref{cond:kernel_sep}, with an additional bound on the $(2,\infty)$ norm of the
separator.
A concrete construction of $\barU$ using \Cref{cond:kernel_sep} is given in
\cref{eq:baru}.

Let $\mu_{\mathcal{N}}$ denote the Gaussian measure on $\mathbb{R}^d$, given by
the Gaussian density with respect to the Lebesgue measure on $\mathbb{R}^d$.
We consider the following Hilbert space
\begin{align*}
    \mathcal{H}:=\setdef{w:\mathbb{R}^d\to \mathbb{R}^d}{\int\|w(z)\|_2^2\dif\mu_{\mathcal{N}}(z)<\infty}.
\end{align*}
For any $x\in\R^d$, define $\phi_x\in\cH$ by
\begin{align*}
    \phi_x(z):=x\1\sbr{\ip{z}{x}>0},
\end{align*}
and particularly define $\phi_i:=\phi_{x_i}$ for the training input $x_i$.

\begin{assumption}\label{cond:kernel_sep}
    There exists $\barv\in \mathcal{H}$ and $\gamma>0$, such that
    $\enVert{\barv(z)}_2\le 1$ for any $z\in \mathbb{R}^d$, and for any
    $1\le i\le n$,
    \begin{align*}
        y_i\ip{\barv}{\phi_i}_\cH:=y_i\int\ip{\barv(z)}{\phi_i(z)}\dif\mu_{\cN}(z)\ge\gamma.
    \end{align*}
\end{assumption}

As discussed in \Cref{sec:sep}, the space $\cH$ is the reproducing kernel
Hilbert space (RKHS) induced by the infinite-width NTK with respect to $W$, and
$\phi_x$ maps $x$ into $\cH$.
\Cref{cond:kernel_sep} supposes that the induced training set
$\{(\phi_i,y_i)\}_{i=1}^n$ can be separated by some $\barv\in\cH$, with an
additional bound on $\enVert{\barv(z)}_2$ which is crucial in our analysis.
It is also possible to give a dual characterization of the separation margin
(cf. \cref{eq:margin_ntk}), which also allows us to show that
\Cref{cond:kernel_sep} always holds when there are no parallel inputs (cf.
\Cref{fact:barv_kernel}).
However, it is often more convenient to construct $\barv$ directly; see
\Cref{sec:sep} for some examples.

With \Cref{cond:kernel_sep}, we state our main empirical risk result.
\begin{theorem}\label{fact:erm}
    Under \Cref{cond:kernel_sep}, given any risk target $\epsilon\in(0,1)$ and
    any $\delta\in(0,1/3)$, let
    \begin{align*}
        \lambda:=\frac{\sqrt{2\ln(4n/\delta)}+\ln(4/\epsilon)}{\gamma/4},\quad\textrm{and}\quad M:=\frac{4096\lambda^2}{\gamma^6}.
    \end{align*}
    Then for any $m\ge M$ and any constant step size $\eta\le1$, with
    probability $1-3\delta$ over the random initialization,
    \begin{align*}
        \frac{1}{T}\sum_{t<T}^{}\hcR(\W{t})\le\epsilon,\quad\textrm{where}\quad T:=\lceil\nicefrac{2\lambda^2}{\eta\epsilon}\rceil.
    \end{align*}
    Moreover for any $0\le t<T$ and any $1\le s\le m$,
    \begin{align*}
        \enVert{\w{s}{t}-\w{s}{0}}_2\le\frac{4\lambda}{\gamma\sqrt{m}}.
    \end{align*}
\end{theorem}

While the number of hidden units required by prior work all have a polynomial
dependency on $n$, $1/\delta$ or $1/\epsilon$, \Cref{fact:erm} only requires
$m=\Omega\del{\ln(n/\delta)+\ln(1/\epsilon)^2}$.
The required width has a polynomial dependency on $1/\gamma$, which is an
adaptive quantity: while $1/\gamma$ can be $\poly(n)$ for random labels (cf.
\Cref{fact:random_label}), it can be $\polylog(n)$ when there is a strong
feature-label relationship, for example on the noisy 2-XOR data introduced in
\citep{wei_reg} (cf. \Cref{fact:2_xor_gamma}).
Moreover, we show in \Cref{fact:ntk_lb} that if we want
$\cbr[1]{\del{\nabla f_i(\W{0}),y_i}}_{i=1}^n$ to be separable, which is the
starting point of an NTK-style analysis, the width has to depend polynomially
on $1/\gamma$.

In the rest of \Cref{sec:erm}, we give a proof sketch of \Cref{fact:erm}.
The full proof is given in \Cref{app_sec:erm}.

\subsection{Properties at initialization}

In this subsection, we give some nice properties of random initialization.

Given an initialization $(\W{0},a)$, for any $1\le s\le m$, define
\begin{align}\label{eq:baru}
    \baru_s:=\frac{1}{\sqrt{m}}a_s\barv(\w{s}{0}),
\end{align}
where $\barv$ is given by \Cref{cond:kernel_sep}.
Collect $\baru_s$ into a matrix $\barU\in \mathbb{R}^{m\times d}$.
It holds that $\enVert{\baru_s}_2\le1/\sqrt{m}$, and $\enVert[1]{\barU}_F\le1$.

\newpage

\Cref{fact:ntk_to_ek} ensures that with high probability $\barU$ has a positive
margin at initialization.
\begin{lemma}\label{fact:ntk_to_ek}
    Under \Cref{cond:kernel_sep}, given any $\delta\in(0,1)$ and any
    $\epsilon_1\in(0,\gamma)$, if $m\ge\del{2\ln(n/\delta)}/\epsilon_1^2$, then
    with probability $1-\delta$, it holds simultaneously for all $1\le i\le n$
    that
    \begin{align*}
        y_if_i^{(0)}\del[1]{\barU}=y_i\ip{\nabla f_i(\W{0})}{\barU}\ge\gamma-\sqrt{\frac{2\ln(n/\delta)}{m}}\ge\gamma-\epsilon_1.
    \end{align*}
\end{lemma}

For any $W$, any $\epsilon_2>0$, and any $1\le i\le n$, define
\begin{align*}
    \alpha_i(W,\epsilon_2)=\frac{1}{m}\sum_{s=1}^{m}\1\sbr{\envert{\ip{w_s}{x_i}}\le\epsilon_2}.
\end{align*}
\Cref{fact:stable_act} controls $\alpha_i(\W{0},\epsilon_2)$.
It will help us show that $\barU$ has a good margin during the training process.
\begin{lemma}\label{fact:stable_act}
    Under the condition of \Cref{fact:ntk_to_ek}, for any $\epsilon_2>0$, with
    probability $1-\delta$, it holds simultaneously for all $1\le i\le n$ that
    \begin{align*}
        \alpha_i\del{\W{0},\epsilon_2}\le\sqrt{\frac{2}{\pi}}\epsilon_2+\sqrt{\frac{\ln(n/\delta)}{2m}}\le \epsilon_2+\frac{\epsilon_1}{2}.
    \end{align*}
\end{lemma}

Finally, \Cref{fact:init_size} controls the output of the network at
initialization.
\begin{lemma}\label{fact:init_size}
    Given any $\delta\in(0,1)$, if $m\ge25\ln(2n/\delta)$, then with probability
    $1-\delta$, it holds simultaneously for all $1\le i\le n$ that
    \begin{align*}
        \envert{f(x_i;\W{0},a)}\le\sqrt{2\ln\del{4n/\delta}}.
    \end{align*}
\end{lemma}

\subsection{Convergence analysis of gradient descent}

We analyze gradient descent in this subsection.
First, define
\begin{align*}
    \hcQ(W):=\frac{1}{n}\sum_{i=1}^{n}-\ell'\del{y_if_i(W)}.
\end{align*}
We have the following observations.
\begin{itemize}
    \item For any $W$ and any $1\le s\le m$,
    $\enVert{\partial f_i/\partial w_s}_2\le1/\sqrt{m}$, and thus
    $\enVert{\nabla f_i(W)}_F\le1$.
    Therefore by the triangle inequality, $\enVert{\hnR(W)}_F\le\hcQ(W)$.

    \item The logistic loss satisfies $0\le-\ell'\le1$, and thus
    $0\le\hcQ(W)\le1$.

    \item The logistic loss satisfies $-\ell'\le\ell$, and thus
    $\hcQ(W)\le\hcR(W)$.
\end{itemize}

The quantity $\hcQ$ first appeared in the perceptron analysis \citep{novikoff}
for the ReLU loss, and has also been analyzed in prior work \citep{min_norm,cao_deep_gen_1,nitanda_refined}.
In this work, $\hcQ$ specifically helps us prove the following result, which
plays an important role in obtaining a width which only depends on
$\polylog(1/\epsilon)$.
\begin{lemma}\label{fact:squared_dist}
    For any $t\ge0$ and any $\barW$, if $\eta_t\le1$, then
    \begin{align*}
        \eta_t\hcR(\W{t})\le\enVert{\W{t}-\barW}_F^2-\enVert{\W{t+1}-\barW}_F^2+2\eta_t\hcR^{(t)}\del[1]{\barW}.
    \end{align*}
    Consequently, if we use a constant step size $\eta\le1$ for $0\le\tau<t$,
    then
    \begin{align*}
        \eta\del{\sum_{\tau<t}^{}\hcR(\W{\tau})}+\enVert{\W{t}-\barW}_F^2\le\enVert{\W{0}-\barW}_F^2+2\eta\del{\sum_{\tau<t}^{}\hcR^{(\tau)}\del[1]{\barW}}.
    \end{align*}
\end{lemma}
The proof of \Cref{fact:squared_dist} starts from the standard iteration
guarantee:
\begin{align*}
    \enVert{\W{t+1}-\barW}_F^2=\enVert{\W{t}-\barW}_F^2-2\eta_t\ip{\hnR(\W{t})}{\W{t}-\barW}+\eta_t^2\enVert{\hnR(\W{t})}_F^2.
\end{align*}
We can then handle the inner product term using the convexity of $\ell$ and
homogeneity of ReLU, and control $\|\hnR(\W{t})\|_F^2$ by $\hcR(\W{t})$ using
the above properties of $\hcQ(\W{t})$.
\Cref{fact:squared_dist} is similar to
\citep[Fact D.4 and Claim D.5]{allen_resnet}, where the squared loss is
considered.

Using \Cref{fact:ntk_to_ek,fact:stable_act,fact:init_size,fact:squared_dist}, we
can prove \Cref{fact:erm}.
Below is a proof sketch; the full proof is given in \Cref{app_sec:erm}.
\begin{enumerate}
    \item We first show that as long as
    $\|\w{s}{t}-\w{s}{0}\|_2\le4\lambda/(\gamma\sqrt{m})$ for all
    $1\le s\le m$, it holds that
    $\hcR^{(t)}\del[1]{\W{0}+\lambda\barU}\le\epsilon/4$.
    To see this, let us consider $\hcR^{(0)}$ first.
    For any $1\le i\le n$, \Cref{fact:init_size} ensures that
    $|\langle\nf_i(\W{0}),\W{0}\rangle|$ is bounded, while \Cref{fact:ntk_to_ek}
    ensures that $\big\langle\nf_i(\W{0}),\barU\big\rangle$ is concentrated
    around $\gamma$ with a large width.
    As a result, with the chosen $\lambda$ in \Cref{fact:erm}, we can show that
    $\big\langle\nf_i(\W{0}),\W{0}+\lambda\barU\big\rangle$ is large, and
    $\hcR^{(0)}(\W{0}+\lambda\barU)$ is small due to the exponential tail of the
    logistic loss.
    To further handle $\hcR^{(t)}$, we use a standard NTK argument to control
    $\big\langle\nf_i(\W{t})-\nf_i(\W{0}),\W{0}+\lambda\barU\big\rangle$ under
    the condition
    that $\|\w{s}{t}-\w{s}{0}\|_2\le4\lambda/(\gamma\sqrt{m})$.

    \item We then prove by contradiction that the above
    bound on $\|\w{s}{t}-\w{s}{0}\|_2$ holds for at least the first $T$ iterations.
    The key observation is that as long as
    $\hcR^{(t)}(\W{0}+\lambda\barU)\le\epsilon/4$, we can use it and
    \Cref{fact:squared_dist} to control $\sum_{\tau<t}^{}\hcQ(\W{\tau})$, and
    then just invoke
    $\|\w{s}{t}-\w{s}{0}\|_2\le\eta \sum_{\tau<t}^{}\hcQ(\W{\tau})/\sqrt{m}$.

    The quantity $\sum_{\tau<t}^{}\hcQ(\W{\tau})$ has also been considered in
    prior work \citep{cao_deep_gen_1,nitanda_refined}, where it is bounded by
    $\sqrt{t}\sqrt{\sum_{\tau<t}^{}\hcQ(\W{\tau})^2}$ using the Cauchy-Schwarz
    inequality, which introduces a $\sqrt{t}$ factor.
    To make the required width depend only on
    $\polylog(1/\epsilon)$, we also need an upper bound on
    $\sum_{\tau<t}^{}\hcQ(\W{\tau})$ which depends only on
    $\polylog(1/\epsilon)$.
    Since the above analysis results in a $\sqrt{t}$ factor, and in our case
    $\Omega(1/\epsilon)$ steps are needed, it is unclear how to get a
    $\polylog(1/\epsilon)$ width using the analysis in
    \citep{cao_deep_gen_1,nitanda_refined}.
    By contrast, using \Cref{fact:squared_dist}, we can show that
    $\sum_{\tau<t}^{}\hcQ(\W{\tau})\le4\lambda/\gamma$, which only depends on
    $\ln(1/\epsilon)$.

    \item The claims of \Cref{fact:erm} then follow directly from the above two
    steps and \Cref{fact:squared_dist}.
\end{enumerate}

\section{Generalization}\label{sec:gen}

To get a generalization bound, we naturally extend \Cref{cond:kernel_sep} to the
following assumption.
\begin{assumption}\label{cond:kernel_sep_gen}
    There exists $\barv\in \mathcal{H}$ and $\gamma>0$, such that
    $\enVert{\barv(z)}_2\le 1$ for any $z\in \mathbb{R}^d$, and
    \begin{align*}
        y\int\ip{\barv(z)}{x}\1\sbr{\ip{z}{x}>0}\dif\mu_{\mathcal{N}}(z)\ge\gamma
    \end{align*}
    for almost all $(x,y)$ sampled from the data distribution $\cD$.
\end{assumption}
The above assumption is also made in \citep{nitanda_refined} for smooth
activations.
\citep{cao_deep_gen_1} make a similar separability assumption, but in the RKHS
induced by the second layer $a$; by contrast, \Cref{cond:kernel_sep_gen} is on
separability in the RKHS induced by the first layer $W$.

Here is our test error bound with \Cref{cond:kernel_sep_gen}.
\begin{theorem}\label{fact:gen}
    Under \Cref{cond:kernel_sep_gen}, given any $\epsilon\in(0,1)$ and any
    $\delta\in(0,1/4)$, let $\lambda$ and $M$ be given as in
    \Cref{fact:erm}:
    \begin{align*}
        \lambda:=\frac{\sqrt{2\ln(4n/\delta)}+\ln(4/\epsilon)}{\gamma/4},\quad\textrm{and}\quad M:=\frac{4096\lambda^2}{\gamma^6}.
    \end{align*}
    Then for any $m\ge M$ and any constant step size $\eta\le1$, with probability
    $1-4\delta$ over the random initialization and data sampling,
    \begin{align*}
        P_{(x,y)\sim\cD}\del{yf(x;\W{k},a)\le0}\le2\epsilon+\frac{16\del{\sqrt{2\ln(4n/\delta)}+\ln(4/\epsilon)}}{\gamma^2\sqrt{n}}+6\sqrt{\frac{\ln(2/\delta)}{2n}},
    \end{align*}
    where $k$ denotes the step with the minimum empirical risk before
    $\lceil\nicefrac{2\lambda^2}{\eta\epsilon}\rceil$.
\end{theorem}

Below is a direct corollary of \Cref{fact:gen}.
\begin{corollary}\label{fact:gen_cor}
    Under \Cref{cond:kernel_sep_gen}, given any $\epsilon,\delta\in(0,1)$, using
    a constant step size no larger than $1$ and let
    \begin{align*}
        n=\wtOmega\del{\frac{1}{\gamma^4\epsilon^2}},\quad\textrm{and}\quad m=\Omega\del{\frac{\ln(n/\delta)+\ln(1/\epsilon)^2}{\gamma^8}},
    \end{align*}
    it holds with probability $1-\delta$ that
    $P_{(x,y)\sim\cD}\del{yf(x;\W{k},a)\le0}\le\epsilon$, where $k$ denotes the
    step with the minimum empirical risk in the first
    $\wtTheta(\nicefrac{1}{\gamma^2\epsilon})$ steps.
\end{corollary}

The proof of \Cref{fact:gen} uses the sigmoid mapping
$-\ell'(z)=e^{-z}/(1+e^{-z})$, the empirical average $\hcQ(\W{k})$, and the
corresponding population average
$\cQ(\W{k}):=\mathbb{E}_{(x,y)\sim\cD}\sbr{-\ell'\del{yf(x;\W{k},a)}}$.
As noted in \citep{cao_deep_gen_1}, because
$P_{(x,y)\sim\cD}\del{yf(x;\W{k},a)\le0}\le2\cQ(\W{k})$, it is enough to control
$\cQ(\W{k})$.
As $\hcQ(\W{k})$ is controlled by \Cref{fact:erm}, it is enough to control
the generalization error $\cQ(\W{k})-\hcQ(\W{k})$.
Moreover, since $-\ell'$ is supported on $[0,1]$ and $1$-Lipschitz, it is enough
to bound the Rademacher complexity of the function space explored by gradient
descent.
Invoking the bound on $\enVert{\W{k}^\top-\W{0}^\top}_{2,\infty}$ finishes the
proof.
The proof details are given in \Cref{app_sec:gen}.

\begin{remark}\label{fact:gen_smooth}
    To get \Cref{fact:gen}, we use a Lipschitz-based Rademacher complexity
    bound.
    One can also use a smoothness-based Rademacher complexity bound
    \citep[Theorem 1]{nati_smooth} and get a sample complexity
    $\wtO(\nicefrac{1}{\gamma^4\epsilon})$.
    However, the bound will become complicated and some large constant will be
    introduced.
    It is an interesting open question to give a clean analysis based on
    smoothness.
\end{remark}

\section{Stochastic gradient descent}\label{sec:sgd}

There are some different formulations of SGD.
In this section, we consider SGD with an online oracle.
We randomly sample $\W{0}$ and $a$, and fix $a$ during training.
At step $i$, a data example $(x_i,y_i)$ is sampled from the data distribution.
We still let $f_i(W):=f(x_i;W,a)$, and perform the following update
\begin{align*}
    \W{i+1}:=\W{i}-\eta_i\ell'\del{y_if_i(\W{i})}y_i\nf_i(\W{i}).
\end{align*}
Note that here $i$ starts from $0$.

Still with \Cref{cond:kernel_sep_gen}, we show the following result.
\begin{theorem}\label{fact:sgd}
    Under \Cref{cond:kernel_sep_gen}, given any $\epsilon,\delta\in(0,1)$, using
    a constant step size and $m=\Omega\del{\nicefrac{\del{\ln(1/\delta)+\ln(1/\epsilon)^2}}{\gamma^8}}$, it holds with probability $1-\delta$ that
    \begin{align*}
        \frac{1}{n}\sum_{i=1}^{n}P_{(x,y)\sim\cD}\del{yf(x;\W{i},a)\le0}\le\epsilon,\quad\textrm{for}\quad n=\wtTheta(\nicefrac{1}{\gamma^2\epsilon}).
    \end{align*}
\end{theorem}

Below is a proof sketch of \Cref{fact:sgd}; the complete proof is given in
\Cref{app_sec:sgd}.
For any $i$ and $W$, define
\begin{align*}
    \cR_i(W):=\ell\del{y_i\ip{\nf_i(\W{i})}{W}},\quad\textrm{and}\quad \cQ_i(W):=-\ell'\del{y_i\ip{\nf_i(\W{i})}{W}}.
\end{align*}
Due to homogeneity, it holds that $\cR_i(\W{i})=\ell\del{y_if_i(\W{i})}$ and
$\cQ_i(\W{i})=-\ell'\del{y_if_i(\W{i})}$.

\newpage

The first step is an extension of \Cref{fact:squared_dist} to the SGD setting,
with a similar proof.
\begin{lemma}\label{fact:squared_dist_sgd}
    With a constant step size $\eta\le1$, for any $\barW$ and any $i\ge0$,
    \begin{align*}
        \eta\del{\sum_{t<i}^{}\cR_t(\W{t})}+\enVert{\W{i}-\barW}_F^2\le\enVert{\W{0}-\barW}_F^2+2\eta\del{\sum_{t<i}^{}\cR_t\del[1]{\barW}}.
    \end{align*}
\end{lemma}

With \Cref{fact:squared_dist_sgd}, we can also extend \Cref{fact:erm} to the SGD
setting and get a bound on $\sum_{i<n}^{}\cQ_i(\W{i})$, using a similar proof.
To further get a bound on the cumulative population risk
$\sum_{i<n}^{}\cQ(\W{i})$, the key observation is that
$\sum_{i<n}^{}\del{\cQ(\W{i})-\cQ_i(\W{i})}$ is a martingale.
Using a martingale Bernstein bound, we prove the following lemma; applying it
finishes the proof of \Cref{fact:sgd}.
\begin{lemma}\label{fact:gen_sgd}
    Given any $\delta\in(0,1)$, with probability $1-\delta$,
    \begin{align*}
        \sum_{t<i}^{}\cQ(\W{t})\le4 \sum_{t<i}^{}\cQ_t(\W{t})+4\ln\del{\frac{1}{\delta}}.
    \end{align*}
\end{lemma}

\section{On separability}\label{sec:sep}

In this section we give some discussion on \Cref{cond:kernel_sep}, the
separability of the NTK.
The proofs are all given in \Cref{app_sec:sep}.

Given a training set $\cbr{(x_i,y_i)}_{i=1}^n$, the linear kernel is defined as
$K_0(x_i,x_j):=\ip{x_i}{x_j}$.
The maximum margin achievable by a linear classifier is given by
\begin{align}\label{eq:gamma_lin}
    \gamma_0:=\min_{q\in\Delta_n}\sqrt{\del{q\odot y}^\top K_0\del{q\odot y}}.
\end{align}
where $\Delta_n$ denotes the probability simplex and $\odot$ denotes the
Hadamard product.
In addition to the dual definition \cref{eq:gamma_lin}, when $\gamma_0>0$ there
also exists a maximum margin classifier $\baru$ which gives a primal
characterization of $\gamma_0$: it holds that $\|\baru\|_2=1$ and
$y_i\ip{\baru}{x_i}\ge\gamma_0$ for all $i$.

In this paper we consider another kernel, the infinite-width NTK with respect to
the first layer:
\begin{align*}
    K_1\del{x_i,x_j} & :=\mathbb{E}\sbr{\frac{\partial f(x_i;\W{0},a)}{\partial \W{0}},\frac{\partial f(x_j;\W{0},a)}{\partial \W{0}}} \\
     & =\mathbb{E}_{w\sim \mathcal{N}(0,I_d)}\sbr[2]{\Big\langle x_i\1\sbr{\langle x_i,w\rangle>0},x_j\1\sbr{\langle x_j,w\rangle>0}\Big\rangle}=\langle\phi_i,\phi_j\rangle_\cH.
\end{align*}
Here $\phi$ and $\cH$ are defined at the beginning of \Cref{sec:erm}.
Similar to the dual definition of $\gamma_0$, the margin given by $K_1$ is
defined as
\begin{align}\label{eq:margin_ntk}
    \gamma_1:=\min_{q\in\Delta_n}\sqrt{\del{q\odot y}^\top K_1\del{q\odot y}}.
\end{align}
We can also give a primal characterization of $\gamma_1$ when it is positive.
\begin{proposition}\label{fact:barv_kernel}
    If $\gamma_1>0$, then there exists $\hatv\in \mathcal{H}$ such that
    $\enVert{\hatv}_{\mathcal{H}}=1$, and
    $y_i\ip{\hatv}{\phi_i}_\cH\ge\gamma_1$ for any $1\le i\le n$.
    Additionally $\enVert{\hatv(z)}_2\le1/\gamma_1$ for any $z\in \mathbb{R}^d$.
\end{proposition}
The proof is given in \Cref{app_sec:sep}, and uses the Fenchel duality theory.
Using the upper bound $\enVert{\hatv(z)}_2\le1/\gamma_1$, we can see that
$\gamma_1\hatv$ satisfies \Cref{cond:kernel_sep} with $\gamma\ge\gamma_1^2$.
However, such an upper bound $\enVert{\hatv(z)}_2\le1/\gamma_1$ might be too
loose, which leads to a bad rate.
In fact, as shown later, in some cases we can construct $\barv$ directly which
satisfies \Cref{cond:kernel_sep} with a large $\gamma$.
For this reason, we choose to make \Cref{cond:kernel_sep} instead of assuming a
positive $\gamma_1$.

However, we can use $\gamma_1$ to show that \Cref{cond:kernel_sep} always holds
when there are no parallel inputs.
\citet[Corollary I.2]{oymak_moderate} prove that if for any two feature vectors
$x_i$ and $x_j$, we have $\|x_i-x_j\|_2\ge\theta$ and $\|x_i+x_j\|_2\ge\theta$
for some $\theta>0$, then the minimum eigenvalue of $K_1$ is at least
$\theta/(100n^2)$.
For arbitrary labels $y\in\{-1,+1\}^n$, since
$\enVert{q\odot y}_2\ge1/\sqrt{n}$, we have the worst case bound
$\gamma_1^2\ge\nicefrac{\theta}{100n^3}$.
A direct improvement of this bound is $\nicefrac{\theta}{100n_S^3}$, where $n_S$
denotes the number of support vectors, which could be much smaller than $n$ with
real world data.

On the other hand, given any training set $\cbr{(x_i,y_i)}_{i=1}^n$ which may
have a large margin, replacing $y$ with random labels would destroy the margin,
which is what should be expected.
\begin{proposition}\label{fact:random_label}
    Given any training set $\cbr{(x_i,y_i)}_{i=1}^n$, if the true labels $y$
    are replaced with random labels
    $\epsilon\sim \mathrm{unif}\del{\{-1,+1\}^n}$, then with probability $0.9$
    over the random labels, it holds that $\gamma_1\le1/\sqrt{20n}$.
\end{proposition}

Although the above bounds all have a polynomial dependency on $n$, they hold for
arbitrary or random labels, and thus do not assume any relationship between the
features and labels.
Next we give some examples where there is a strong feature-label relationship,
and thus a much larger margin can be proved.

\subsection{The linearly separable case}\label{app_sec:linear}

Suppose the data distribution is linearly separable with margin $\gamma_0$:
there exists a unit vector $\baru$ such that $y\ip{\baru}{x}\ge\gamma_0$ almost
surely.
Then we can define $\barv(z):=\baru$ for any $z\in\R^d$.
For almost all $(x,y)$, we have
\begin{align*}
    y\int\ip{\barv(z)}{x}\1\sbr{\ip{z}{x}>0}\dif\mu_{\cN}(z) & =\int y\ip{\baru}{x}\1\sbr{\ip{z}{x}>0}\dif\mu_{\cN}(z) \\
     & \ge \gamma\int\1\sbr{\ip{z}{x}>0}\dif\mu_{\cN}(z) \\
     & =\frac{\gamma_0}{2},
\end{align*}
and thus \Cref{cond:kernel_sep} holds with $\gamma=\gamma_0/2$.

\subsection{The noisy 2-XOR distribution}\label{app_sec:2xor}

We consider the noisy 2-XOR distribution introduced in \citep{wei_reg}.
It is the uniform distribution over the following $2^d$
points:
\begin{align*}
    (x_1,x_2,y,x_3,\ldots,x_d)\in & \cbr[3]{\del{\frac{1}{\sqrt{d-1}},0,1},\del{0,\frac{1}{\sqrt{d-1}},-1},\del{\frac{-1}{\sqrt{d-1}},0,1},\del{0,\frac{-1}{\sqrt{d-1}},-1}} \\
     & \times\cbr{\frac{-1}{\sqrt{d-1}},\frac{1}{\sqrt{d-1}}}^{d-2}.
\end{align*}
The factor $\nicefrac{1}{\sqrt{d-1}}$ ensures that $\|x\|_2=1$, and $\times$
above denotes the Cartesian product.
Here the label $y$ only depends on the first two coordinates of the input $x$.

To construct $\barv$, we first decompose $\R^2$ into four regions:
\begin{align*}
    A_1 & :=\setdef{(z_1,z_2)}{z_1\ge0,|z_1|\ge|z_2|}, \\
    A_2 & :=\setdef{(z_1,z_2)}{z_2>0,|z_1|<|z_2|}, \\
    A_3 & :=\setdef{(z_1,z_2)}{z_1\le0,|z_1|\ge|z_2|}\setminus\{(0,0)\}, \\
    A_4 & :=\setdef{(z_1,z_2)}{z_2<0,|z_1|<|z_2|}.
\end{align*}
Then $\barv$ can de defined as follows.
It only depends on the first two coordinates of $z$.
\begin{equation} \label{eq:2_xor_barv}
    \barv(z):=
    \begin{dcases}
        (1,0,0,\ldots,0) & \textrm{if }(z_1,z_2)\in A_1, \\
        (0,-1,0,\ldots,0) & \textrm{if }(z_1,z_2)\in A_2, \\
        (-1,0,0,\ldots,0) & \textrm{if }(z_1,z_2)\in A_3, \\
        (0,1,0,\ldots,0) & \textrm{if }(z_1,z_2)\in A_4.
    \end{dcases}
\end{equation}

The following result shows that $\gamma=\Omega(1/d)$.
Note that $n$ could be as large as $2^d$, in which case $\gamma$ is basically
$O\del{1/\ln(n)}$.
\begin{proposition}\label{fact:2_xor_gamma}
    For any $(x,y)$ sampled from the noisy 2-XOR distribution and any $d\ge3$,
    it holds that
    \begin{align*}
        y\int\ip{\barv(z)}{x}\1\sbr{\ip{z}{x}>0}\dif\mu_{\cN}(z)\ge \frac{1}{60d}.
    \end{align*}
\end{proposition}

We can prove two other interesting results for the noisy 2-XOR data.

\paragraph{The width needs a $\poly(1/\gamma)$ dependency for initial separability.}

The first step of an NTK analysis is to show that
$\cbr[1]{\del{\nabla f_i(\W{0}),y_i}}_{i=1}^n$ is separable.
\Cref{fact:ntk_lb} gives an example where
$\cbr[1]{\del{\nabla f_i(\W{0}),y_i}}_{i=1}^n$ is nonseparable when the network
is narrow.
\begin{proposition}\label{fact:ntk_lb}
    Let $D=\{(x_i,y_i)\}_{i=1}^4$ denote an arbitrary subset of the noisy 2-XOR
    dataset such that $x_i$'s have the same last $(d-2)$ coordinates.
    For any $d\ge20$, if $m\le\sqrt{d-2}/4$, then with probability $1/2$ over
    the random initialization of $\W{0}$, for any weights $V\in\R^{m\times d}$,
    it holds that $y_i\ip{V}{\nabla f_i(\W{0})}\le0$ for at least one
    $i\in\{1,2,3,4\}$.
\end{proposition}

For the noisy 2-XOR data, the separator $\barv$ given by \cref{eq:2_xor_barv}
has margin $\gamma=\Omega(1/d)$, and $1/\gamma=O(d)$.
As a result, if we want $\cbr[1]{\del{\nabla f_i(\W{0}),y_i}}_{i=1}^n$ to be
separable, the width has to be $\Omega(1/\sqrt{\gamma})$.
For a smaller width, gradient descent might still be able to solve the problem,
but a beyond-NTK analysis would be needed.

\paragraph{A tight sample complexity upper bound for the infinite-width NTK.}

\citep{wei_reg} give a $d^2$ sample complexity lower bound for any NTK
classifier on the noisy 2-XOR data.
It turns out that $\gamma$ could give a \emph{matching} sample complexity upper
bound for the NTK and SGD.

\citep{wei_reg} consider the infinite-width NTK with respect to both layers.
For the first layer, the infinite-width NTK $K_1$ is defined in \Cref{sec:sep},
and the corresponding RKHS $\cH$ and RKHS mapping $\phi$ is defined in
\Cref{sec:erm}.
For the second layer, the infinite width NTK is defined by
\begin{align*}
    K_2\del{x_i,x_j} & :=\mathbb{E}\sbr{\frac{\partial f(x_i;\W{0},a)}{\partial a},\frac{\partial f(x_j;\W{0},a)}{\partial a}} \\
     & =\mathbb{E}_{w\sim \mathcal{N}(0,I_d)}\sbr[2]{\sigma\del{\langle w,x_i\rangle}\sigma\del{\langle w,x_j\rangle}}.
\end{align*}
The corresponding RKHS $\cK$ and inner product $\langle w_1,w_2\rangle_\cK$ are
given by
\begin{align*}
    \cK:=\setdef{w:\R^d\to\R}{\int w(z)^2\dif\mu_{\mathcal{N}}(z)<\infty},\quad\textrm{and}\quad \langle w_1,w_2\rangle_\cK=\int w_1(z)w_2(z)\dif\mu_{\mathcal{N}}(z).
\end{align*}
Given any $x\in\R^d$, it is mapped into $\psi_x\in\cK$, where
$\psi_x(z):=\sigma\del{\langle z,x\rangle}$.
It holds that $K_2(x_i,x_j)=\langle\psi_{x_i},\psi_{x_j}\rangle_\cK$.
The infinite-width NTK with respect to both layers is just $K_1+K_2$.
The corresponding RHKS is just $\cH\times\cK$ with the inner product
\begin{align*}
    \langle (v_1,w_1),(v_2,w_2)\rangle_{\cH\times\cK}=\langle v_1,v_2\rangle_\cH+\langle w_1,w_2\rangle_\cK.
\end{align*}

The classifier $\barv$ considered in \cref{eq:2_xor_barv} has a unit norm (i.e.,
$\enVert{\barv}_\cH=1$) and margin $\gamma$ on the space $\cH$.
On $\cH\times\cK$, it is enough to consider $(\barv,0)$, which also has a unit
norm and margin $\gamma$.
Since the infinite-width NTK model is a linear model in $\cH\times\cK$,
\citep[Lemma 2.5]{min_norm} can be used to show that SGD on the RKHS
$\cH\times\cK$ could obtain a test error of $\epsilon$ with a sample complexity
of $\wtO(\nicefrac{1}{\gamma^2\epsilon})$.
(The analysis in \citep{min_norm} is done in $\R^d$, but it still works with a
well-defined inner product.)
Since $\gamma=\Omega(1/d)$, to achieve a constant test accuracy we need
$\wtO(d^2)$ samples.
This mathces (up to logarithmic factors) the sample complexity lower bound of
$d^2$ given by \citet{wei_reg}.

\section{Open problems}\label{sec:open}

In this paper, we analyze gradient descent on a two-layer network in the NTK
regime, where the weights stay close to the initialization.
It is an interesting open question if gradient descent learns something beyond
the NTK, after the iterates move far enough from the initial weights.
It is also interesting to extend our analysis to other architectures, such as
multi-layer networks,
convolutional networks, and residual networks.
Finally, in this paper we only discuss binary classification; it is interesting
to see if it is possible to get similar results for other tasks, such as
regression.

\subsection*{Acknowledgements}

The authors are grateful for support from the NSF under grant IIS-1750051,
and from NVIDIA via a GPU grant.

\bibliography{bib}
\bibliographystyle{plainnat}

\appendix

\section{Omitted proofs from \Cref{sec:erm}}\label{app_sec:erm}

\begin{proof}[Proof of \Cref{fact:ntk_to_ek}]
    By \Cref{cond:kernel_sep}, given any $1\le i\le n$,
    \begin{align*}
        \mu:=\mathbb{E}_{w\sim \mathcal{N}(0,I_d)}\sbr{y_i\ip{\barv(w)}{x_i}\1\sbr{\ip{w}{x_i}>0}}\ge\gamma.
    \end{align*}
    On the other hand,
    \begin{align*}
        y_if_i^{(0)}\del[1]{\barU}=\frac{1}{m}\sum_{s=1}^{m}y_i\ip{\barv(\w{s}{0})}{x_i}\1\sbr{\ip{\w{s}{0}}{x_i}>0}
    \end{align*}
    is the empirical mean of i.i.d. r.v.'s supported on $[-1,+1]$ with mean
    $\mu$.
    Therefore by Hoeffding's inequality, with probability
    $1-\nicefrac{\delta}{n}$,
    \begin{align*}
        y_if_i^{(0)}\del[1]{\barU}-\gamma\ge y_if_i^{(0)}\del[1]{\barU}-\mu\ge-\sqrt{\frac{2\ln(n/\delta)}{m}}.
    \end{align*}
    Applying a union bound finishes the proof.
\end{proof}

\begin{proof}[Proof of \Cref{fact:stable_act}]
    Given any fixed $\epsilon_2$ and $1\le i\le n$,
    \begin{align*}
        \mathbb{E}\sbr{\alpha_i(\W{0},\epsilon_2)}=\bbP\del{\envert{\ip{w}{x_i}}\le\epsilon_2}\le \frac{2\epsilon_2}{\sqrt{2\pi}}=\sqrt{\frac{2}{\pi}}\epsilon_2,
    \end{align*}
    because $\ip{w}{x_i}$ is a standard Gaussian r.v. and the density of
    standard Gaussian has maximum $1/\sqrt{2\pi}$.
    Since $\alpha_i(\W{0},\epsilon_2)$ is the empirical mean of Bernoulli r.v.'s, by
    Hoeffding's inequality, with probability $1-\nicefrac{\delta}{n}$,
    \begin{align*}
        \alpha_i(\W{0},\epsilon_2)\le \mathbb{E}\sbr{\alpha_i(\W{0},\epsilon_2)}+\sqrt{\frac{\ln(n/\delta)}{2m}}\le\sqrt{\frac{2}{\pi}}\epsilon_2+\sqrt{\frac{\ln(n/\delta)}{2m}}.
    \end{align*}
    Applying a union bound finishes the proof.
\end{proof}

To prove \Cref{fact:init_size}, we need the following technical result.

\begin{lemma}\label{fact:norm_after_relu}
    Consider the random vector $X=(X_1,\ldots,X_m)$, where $X_i=\sigma(Z_i)$ for
    some $\sigma:\mathbb{R}\to \mathbb{R}$ that is $1$-Lipschitz, and $Z_i$ are
    i.i.d. standard Gaussian r.v.'s.
    Then the r.v. $\|X\|_2$ is $1$-sub-Gaussian, and thus with probability
    $1-\delta$,
    \begin{align*}
        \|X\|_2-\mathbb{E}\sbr{\|X\|_2}\le\sqrt{2\ln(1/\delta)}.
    \end{align*}
\end{lemma}
\begin{proof}
    Given $a\in \mathbb{R}^m$, define
    \begin{align*}
        f(a)=\sqrt{\sum_{i=1}^{m}\sigma(a_i)^2}=\enVert{\sigma(a)}_2,
    \end{align*}
    where $\sigma(a)$ is obtained by applying $\sigma$ coordinate-wisely to $a$.
    For any $a,b\in \mathbb{R}^m$, by the triangle inequality, we have
    \begin{align*}
        \envert{f(a)-f(b)}=\envert{\,\enVert{\sigma(a)}_2-\enVert{\sigma(b)}_2} & \le\enVert{\sigma(a)-\sigma(b)}_2=\sqrt{\sum_{i=1}^{m}\del{\sigma(a_i)-\sigma(b_i)}^2},
    \end{align*}
    and by further using the $1$-Lipschitz continuity of $\sigma$, we have
    \begin{align*}
        \envert{f(a)-f(b)}\le\sqrt{\sum_{i=1}^{m}\del{\sigma(a_i)-\sigma(b_i)}^2}\le\sqrt{\sum_{i=1}^{m}(a_i-b_i)^2}=\enVert{a-b}_2.
    \end{align*}
    As a result, $f$ is a $1$-Lipschitz continuous function w.r.t. the $\ell_2$
    norm, indeed $f(X)$ is $1$-sub-Gaussian
    and the bound follows by Gaussian concentration
    \citep[Theorem 2.4]{wainwright_note}.
\end{proof}

\begin{proof}[Proof of \Cref{fact:init_size}]
    Given $1\le i\le n$, let $h_i=\sigma(\W{0}x_i)/\sqrt{m}$.
    By \Cref{fact:norm_after_relu}, $\|h_i\|_2$ is sub-Gaussian with variance
    proxy $1/m$, and with probability at least $1-\nicefrac{\delta}{2n}$ over
    $\W{0}$,
    \begin{align*}
        \|h_i\|_2-\mathbb{E}\sbr{\|h_i\|_2}\le\sqrt{\frac{2\ln(2n/\delta)}{m}}\le\sqrt{\frac{2\ln(2n/\delta)}{25\ln(2n/\delta)}}\le1-\frac{\sqrt{2}}{2}.
    \end{align*}
    On the other hand, by Jensen's inequality,
    \begin{align*}
        \mathbb{E}\sbr{\|h_i\|_2}\le\sqrt{\mathbb{E}\sbr{\|h_i\|_2^2}}=\frac{\sqrt{2}}{2}.
    \end{align*}
    As a result, with probability $1-\nicefrac{\delta}{2n}$, it holds that
    $\|h_i\|_2\le1$.
    By a union bound, with probability $1-\nicefrac{\delta}{2}$ over $\W{0}$,
    for all $1\le i\le n$, we have $\|h_i\|_2\le1$.

    For any $\W{0}$ such that the above event holds, and for any $1\le i\le n$,
    the r.v. $\ip{h_i}{a}$ is sub-Gaussian with variance proxy
    $\|h_i\|_2^2\le1$.
    By Hoeffding's inequality, with probability $1-\nicefrac{\delta}{2n}$ over
    $a$,
    \begin{align*}
        \envert{\ip{h_i}{a}}=\envert{f(x_i;\W{0},a)}\le\sqrt{2\ln\del{4n/\delta}}.
    \end{align*}
    By a union bound, with probability $1-\nicefrac{\delta}{2}$ over $a$,
    for all $1\le i\le n$, we have
    $\envert{f(x_i;\W{0},a)}\le\sqrt{2\ln\del{4n/\delta}}$.

    The probability that the above events all happen is at least
    $(1-\nicefrac{\delta}{2})(1-\nicefrac{\delta}{2})\ge1-\delta$, over $\W{0}$
    and $a$.
\end{proof}

\begin{proof}[Proof of \Cref{fact:squared_dist}]
    We have
    \begin{align}\label{eq:step}
        \enVert{\W{t+1}-\barW}_F^2=\enVert{\W{t}-\barW}_F^2-2\eta_t\ip{\hnR(\W{t})}{\W{t}-\barW}+\eta_t^2\enVert{\hnR(\W{t})}_F^2.
    \end{align}
    The first order term of \cref{eq:step} can be handled using the convexity of
    $\ell$ and homogeneity of ReLU:
    \begin{align}
    \begin{split}\label{eq:1st_order}
        \ip{\hnR(\W{t})}{\W{t}-\barW} & =\frac{1}{n}\sum_{i=1}^{n}\ell'\del{y_if_i(\W{t})}y_i\ip{\nf_i(\W{t})}{\W{t}-\barW} \\
         & =\frac{1}{n}\sum_{i=1}^{n}\ell'\del{y_if_i(\W{t})}\del{y_if_i(\W{t})-y_if_i^{(t)}\del[1]{\barW}} \\
         & \ge \frac{1}{n}\sum_{i=1}^{n}\del{\ell\del{y_if_i(\W{t})}-\ell\del{y_if_i^{(t)}\del[1]{\barW}}}=\hcR(\W{t})-\hcR^{(t)}\del[1]{\barW}.
    \end{split}
    \end{align}

    The second-order term of \cref{eq:step} can be bounded as follows
    \begin{align}\label{eq:2nd_order}
        \eta_t^2\enVert{\hnR(\W{t})}_F^2\le\eta_t^2\hcQ(\W{t})^2\le\eta_t\hcQ(\W{t})\le\eta_t\hcR(\W{t}),
    \end{align}
    because $\enVert{\hnR(\W{t})}_F\le\hcQ(\W{t})$, and $\eta_t,\hcQ(\W{t})\le1$, and
    $\hcQ(\W{t})\le\hcR(\W{t})$.
    Combining \cref{eq:step,eq:1st_order,eq:2nd_order} gives
    \begin{align*}
        \eta_t\hcR(\W{t})\le\enVert{\W{t}-\barW}_F^2-\enVert{\W{t+1}-\barW}_F^2+2\eta_t\hcR^{(t)}\del[1]{\barW}.
    \end{align*}
    Telescoping gives the other claim.
\end{proof}

\begin{proof}[Proof of \Cref{fact:erm}]
    The required width ensures that with probability $1-3\delta$,
    \Cref{fact:ntk_to_ek,fact:stable_act,fact:init_size} hold with
    $\epsilon_1=\gamma^2/8$ and $\epsilon_2=4\lambda/(\gamma\sqrt{m})$.

    Let $t_1$ denote the first step such that there exists $1\le s\le m$ with
    $\enVert{\w{s}{t_1}-\w{s}{0}}_2>4\lambda/(\gamma\sqrt{m})$.
    Therefore for any $0\le t<t_1$ and any $1\le s\le m$, it holds that
    $\enVert{\w{s}{t}-\w{s}{0}}_2\le4\lambda/(\gamma\sqrt{m})$.
    In addition, we let $\barW:=\W{0}+\lambda\barU$.

    We first prove that for any $0\le t<t_1$, it holds that
    $\hcR^{(t)}\del[1]{\barW}\le\epsilon/4$.
    Since $\ln(1+r)\leq r$ for any $r$,
    the logistic satisfies $\ell(z) = \ln(1+\exp(-z))\le\exp(-z)$,
    and it is enough to prove that for any $1\le i\le n$,
    \begin{align*}
        y_i\ip{\nf_i(\W{t})}{\barW}\ge\ln\del{\frac{4}{\epsilon}}.
    \end{align*}
    We will split the left hand side into three terms and control them individually:
    \begin{align}\label{eq:parts}
        y_i\ip{\nf_i(\W{t})}{\barW}=y_i\ip{\nf_i(\W{0})}{\W{0}}+y_i\ip{\nf_i(\W{t})-\nf_i(\W{0})}{\W{0}}+\lambda y_i\ip{\nf_i(\W{t})}{\barU}.
    \end{align}
    \begin{itemize}
        \item The first term of \cref{eq:parts} can be controlled using
        \Cref{fact:init_size}:
        \begin{align}\label{eq:part1}
            \envert{y_i\ip{\nf_i(\W{0})}{\W{0}}}\le\sqrt{2\ln(4n/\delta)}.
        \end{align}

        \item The second term of \cref{eq:parts} can be written as
        \begin{align*}
            y_i\ip{\nf_i(\W{t})-\nf_i(\W{0})}{\W{0}}=y_i \frac{1}{\sqrt{m}}\sum_{s=1}^{m}a_s\del{\1\sbr[1]{\ip{\w{s}{t}}{x_i}>0}-\1\sbr[1]{\ip{\w{s}{0}}{x_i}>0}}\ip{\w{s}{0}}{x_i}.
        \end{align*}
        Let $S_c:=\setdef{s}{\1\sbr[1]{\ip{\w{s}{t}}{x_i}>0}-\1\sbr[1]{\ip{\w{s}{0}}{x_i}>0}\ne0,1\le s\le m}$.
        Note that $s\in S_c$ implies
        \begin{align*}
            \envert{\ip{\w{s}{0}}{x_i}}\le\envert{\ip{\w{s}{t}-\w{s}{0}}{x_i}}\le\enVert{\w{s}{t}-\w{s}{0}}_2\|x_i\|_2=\enVert{\w{s}{t}-\w{s}{0}}_2\le4\lambda/(\gamma\sqrt{m}) = \eps_2.
        \end{align*}
        Therefore \Cref{fact:stable_act} ensures that
        \begin{align*}
          |S_c|
          \le
          \envert{\setdef{ s\ {} }{ \ {}|\ip{\w{s}{0}}{x_i}| \leq \eps_2 }}
          \le m\del{\frac{4\lambda}{\gamma\sqrt{m}}+\frac{\epsilon_1}{2}}=m\del{\frac{4\lambda}{\gamma\sqrt{m}}+\frac{\gamma^2}{16}}.
        \end{align*}
        and thus
        \begin{align}\label{eq:part2}
            \envert{y_i\ip{\nf_i(\W{t})-\nf_i(\W{0})}{\W{0}}}\le \frac{1}{\sqrt{m}}\cdot|S_c|\cdot \frac{4\lambda}{\gamma\sqrt{m}}\le \frac{16\lambda^2}{\gamma^2\sqrt{m}}+\frac{\lambda\gamma}{4}\le \frac{\lambda\gamma}{2},
        \end{align}
        where in the last step we use the condition that
        $m\ge4096\lambda^2/\gamma^6$.

        \item The third term of \cref{eq:parts} can be bounded as follows:
          by \Cref{fact:ntk_to_ek},
        \begin{align*}
            y_i\ip{\nf_i(\W{t})}{\barU} & =y_i\ip{\nf_i(\W{0})}{\barU}+y_i\ip{\nf_i(\W{t})-\nf_i(\W{0})}{\barU} \\
             & \ge\gamma-\epsilon_1+y_i\ip{\nf_i(\W{t})-\nf_i(\W{0})}{\barU}.
        \end{align*}
        In addition,
        \begin{align*}
            y_i\ip{\nf_i(\W{t})-\nf_i(\W{0})}{\barU} & =y_i\frac{1}{m}\sum_{i=1}^{m}\del{\1\sbr{\ip{\w{s}{t}}{x_i}>0}-\1\sbr{\ip{\w{s}{0}}{x_i}>0}}\ip{\barv(\w{s}{0})}{x_i} \\
             & \ge-\frac{1}{m}\cdot|S_c|\ge-\frac{4\lambda}{\gamma\sqrt{m}}-\frac{\epsilon_1}{2}\ge-\frac{\gamma^2}{16}-\frac{\epsilon_1}{2},
        \end{align*}
        where we use $m\ge4096\lambda^2/\gamma^6$.
        Therefore,
        \begin{align}\label{eq:good_margin}
            y_i\ip{\nf_i(\W{t})}{\barU}\ge\gamma-\epsilon_1-\frac{\gamma^2}{16}-\frac{\epsilon_1}{2}=\gamma-\frac{\gamma^2}{4}\ge \frac{3\gamma}{4}.
        \end{align}
    \end{itemize}
    Putting \cref{eq:part1,eq:part2,eq:good_margin} into \cref{eq:parts}, we
    have
    \begin{align*}
        y_i\ip{\nf_i(\W{t})}{\barW}\ge-\sqrt{2\ln\del{\frac{4n}{\delta}}}-\frac{\lambda\gamma}{2}+\frac{3\lambda\gamma}{4}=\frac{\lambda\gamma}{4}-\sqrt{2\ln\del{\frac{4n}{\delta}}}=\ln\del{\frac{4}{\epsilon}},
    \end{align*}
    for the $\lambda$ given in the statement of \Cref{fact:erm}.
    Consequently, for any $0\le t<t_1$, it holds that
    $\hcR^{(t)}\del[1]{\barW}\le\epsilon/4$.

    Let $T:=\lceil\nicefrac{2\lambda^2}{\eta\epsilon}\rceil$.
    The next claim is that $t_1\ge T$.
    To see this, note that \Cref{fact:squared_dist} ensures
    \begin{align*}
        \enVert{\W{t_1}-\barW}_F^2\le\enVert{\W{0}-\barW}_F^2+2\eta\del{\sum_{t<t_1}^{}\hcR^{(t)}\del[1]{\barW}}\le\lambda^2+\frac{\epsilon}{2}\eta t_1.
    \end{align*}
    Suppose $t_1<T$, then we have $t_1\le\nicefrac{2\lambda^2}{\eta\epsilon}$,
    and thus $\enVert{\W{t_1}-\barW}_F^2\le2\lambda^2$.
    As a result, using $\|\barU\|_F\leq 1$ and the definition of $\barW$,
    \begin{align*}
        \sqrt{2}\lambda\ge\enVert{\W{t_1}-\barW}_F\ge\ip{\W{t_1}-\barW}{\barU} & =\ip{\W{t_1}-\W{0}}{\barU}-\ip{\barW-\W{0}}{\barU} \\
         & \ge\ip{\W{t_1}-\W{0}}{\barU}-\lambda.
    \end{align*}
    Moreover, due to \cref{eq:good_margin},
    \begin{align*}
        \ip{\W{t_1}-\W{0}}{\barU}=-\eta \sum_{\tau<t_1}^{}\ip{\hnR(\W{\tau})}{\barU} & =\eta \sum_{\tau<t_1}^{}\frac{1}{n}\sum_{i=1}^{n}-\ell'\del{y_if_i(\W{\tau})}y_i\ip{\nf_i(\W{\tau})}{\barU} \\
         & \ge\eta \sum_{\tau<t_1}^{}\hcQ(\W{\tau})\frac{3\gamma}{4}.
    \end{align*}
    As a result,
    \begin{align*}
        \eta \sum_{\tau<t_1}^{}\hcQ(\W{\tau})\le \frac{4(\sqrt{2}+1)\lambda}{3\gamma}\le \frac{4\lambda}{\gamma}.
    \end{align*}
    Furthermore, by the triangle inequality, for any $1\le s\le m$
    \begin{align}
        \enVert{\w{s}{t}-\w{s}{0}}_2 & \le\eta \sum_{\tau<t}^{}\enVert{\frac{1}{n}\sum_{i=1}^{n}\ell'\del{y_if_i(\W{\tau})}y_i\frac{\partial f_i}{\partial \w{s}{\tau}}}_2 \notag\\
                                       & \le\eta \sum_{\tau<t}^{}\frac{1}{n}\sum_{i=1}^{n}\envert{\ell'\del{y_if_i(\W{\tau})}}\cdot\enVert{\frac{\partial f_i}{\partial \w{s}{\tau}}}_2 \notag\\
         & \le\eta \sum_{\tau<t}^{}\hcQ(\W{\tau})\frac{1}{\sqrt{m}} \notag\\
         & \le\eta \sum_{\tau<t_1}^{}\hcQ(\W{\tau})\frac{1}{\sqrt{m}}\le\frac{4\lambda}{\gamma\sqrt{m}},\label{eq:many_triangle}
    \end{align}
    which contradicts the definition of $t_1$.
    Therefore $t_1\ge T$.

    Now we are ready to prove the claims of \Cref{fact:erm}.
    The bound on $\enVert{\w{s}{t}-\w{s}{0}}_2$ follow by repeating the
    steps in \cref{eq:many_triangle}.
    The risk guarantee follows from \Cref{fact:squared_dist}:
    \begin{align*}
        \frac{1}{T}\sum_{t<T}^{}\hcR(\W{t})\le \frac{\enVert{\W{0}-\barW}_F^2}{\eta T}+\frac{2}{T}\sum_{t<T}^{}\hcR^{(t)}\del[1]{\barW}\le \frac{\epsilon}{2}+\frac{\epsilon}{2}=\epsilon.
    \end{align*}
\end{proof}

\section{Omitted proofs from \Cref{sec:gen}}\label{app_sec:gen}

The proof of \Cref{fact:gen} is based on Rademacher complexity.
Given a sample $S=(z_1,\ldots,z_n)$ (where $z_i=(x_i,y_i)$) and a function class
$\cH$, the Rademacher complexity of $\cH$ on $S$ is defined as
\begin{align*}
    \Rad\del{\cH\circ S}:=\frac{1}{n}\mathbb{E}_{\epsilon\sim\{-1,+1\}^n}\sbr{\sup_{h\in\cH}\sum_{i=1}^{n}\epsilon_ih(z_i)}.
\end{align*}

We will use the following general result.
\begin{lemma}\citep[Theorem 26.5]{shai_shai_book}\label{fact:rc_gen}
    If $h(z)\in[a,b]$, then with probability $1-\delta$,
    \begin{align*}
        \sup_{h\in\cH}\del{\mathbb{E}_{z\sim\cD}\sbr{h(z)}-\frac{1}{n}\sum_{i=1}^{n}h(z_i)}\le2\Rad\del{\cH\circ S}+3(b-a)\sqrt{\frac{\ln(2/\delta)}{2n}}.
    \end{align*}
\end{lemma}

We also need the following contraction lemma.
Consider a feature sample $X=(x_1,\ldots,x_n)$ and a function class $\cF$ on
$X$.
For each $1\le i\le n$, let $g_i:\mathbb{R}\to \mathbb{R}$ denote a
$K$-Lipschitz function.
Let $g\circ \cF$ denote the class of functions which map $x_i$ to
$g_i(f(x_i))$ for some $f\in\cF$.
\begin{lemma}\citep[Lemma 26.9]{shai_shai_book}\label{fact:contraction}
    $\Rad\del{g\circ\cF\circ X}\le K\Rad\del{\cF\circ X}$.
\end{lemma}

To prove \Cref{fact:gen}, we need one more Rademacher complexity bound.
Given a fixed initialization $(\W{0},a)$, consider the following classes:
\begin{align*}
    \cW_{\rho}:=\setdef{W\in \mathbb{R}^{m\times d}}{\enVert{w_{s}-\w{s}{0}}_2\le\rho\textrm{ for any }1\le s\le m},
\end{align*}
and
\begin{align*}
    \quad\cF_{\rho}:=\setdef{x\mapsto f(x;W,a)}{W\in\cW_{\rho}}.
\end{align*}
Given a feature sample $X$, the following \Cref{fact:rc} controls the Rademacher
complexity of $\cF_{\rho}\circ X$.
A similar version was given in \citep[Theorem 43]{percy_note}, and the proof is
similar to the proof of \citep[Theorem 18]{bartlett_mendelson_rademacher} which
also pushes the supremum through and handles each hidden unit separately.
\begin{lemma}\label{fact:rc}
    $\Rad\del{\cF_{\rho}\circ X}\le\rho\sqrt{m/n}$.
\end{lemma}
\begin{proof}[Proof of \Cref{fact:rc}]
    We have
    \begin{align*}
        \mathbb{E}_{\epsilon}\sbr{\sup_{W\in\cW_{\rho}}\sum_{i=1}^{n}\epsilon_if(x_i;W,a)} & =\mathbb{E}_{\epsilon}\sbr{\sup_{W\in\cW_{\rho}}\sum_{i=1}^{n}\epsilon_i \sum_{s=1}^{m}\frac{1}{\sqrt{m}}a_s\sigma\del{\ip{w_s}{x_i}}} \\
         & =\mathbb{E}_{\epsilon}\sbr{\frac{1}{\sqrt{m}}\sup_{W\in\cW_{\rho}}\sum_{s=1}^{m}\sum_{i=1}^{n}\epsilon_ia_s\sigma\del{\ip{w_s}{x_i}}} \\
         & =\mathbb{E}_{\epsilon}\sbr{\frac{1}{\sqrt{m}}\sum_{s=1}^{m}\del{\sup_{\enVert{w_s-\w{s}{0}}_2\le\rho}\sum_{i=1}^{n}\epsilon_ia_s\sigma\del{\ip{w_s}{x_i}}}} \\
         & =\frac{1}{\sqrt{m}}\sum_{i=1}^{m}\mathbb{E}_{\epsilon}\sbr{\sup_{\enVert{w_s-\w{s}{0}}_2\le\rho}\sum_{i=1}^{n}\epsilon_ia_s\sigma\del{\ip{w_s}{x_i}}}.
    \end{align*}
    Note that for any $1\le s\le m$, the mapping $z\mapsto a_s\sigma(z)$ is
    $1$-Lipschitz, and thus \Cref{fact:contraction} gives
    \begin{align*}
        \mathbb{E}_{\epsilon}\sbr{\sup_{W\in\cW_{\rho}}\sum_{i=1}^{n}\epsilon_if(x_i;W,a)} & \le \frac{1}{\sqrt{m}}\sum_{i=1}^{m}\mathbb{E}_{\epsilon}\sbr{\sup_{\enVert{w_s-\w{s}{0}}_2\le\rho}\sum_{i=1}^{n}\epsilon_ia_s\sigma\del{\ip{w_s}{x_i}}} \\
        & \le \frac{1}{\sqrt{m}}\sum_{i=1}^{m}\mathbb{E}_{\epsilon}\sbr{\sup_{\enVert{w_s-\w{s}{0}}_2\le\rho}\sum_{i=1}^{n}\epsilon_i\ip{w_s}{x_i}}.
    \end{align*}
    Invoking the Rademacher complexity of linear classifiers
    \citep[Lemma 26.10]{shai_shai_book} then gives
    \begin{align*}
        \mathrm{Rad}\del{\cF_{\rho}\circ X}=\frac{1}{n}\mathbb{E}_{\epsilon}\sbr{\sup_{W\in\cW_{\rho}}\sum_{i=1}^{n}\epsilon_if(x_i;W,a)}\le \frac{\rho\sqrt{m}}{\sqrt{n}}.
    \end{align*}
\end{proof}

Now we are ready to prove the main generalization result \Cref{fact:gen}.
\begin{proof}
    Fix an initialization $(\W{0},a)$, and let
    $\cH:=\setdef{(x,y)\mapsto-\ell'\del{yf(x)}}{f\in\cF_{\rho}}$.
    Since for any $h\in\cH$ and any $z$, $h(z)\in[0,1]$, \Cref{fact:rc_gen}
    ensures that with probability $1-\delta$ over the data sampling,
    \begin{align*}
        \sup_{h\in\cH}\del{\mathbb{E}_{z\sim\cD}\sbr{h(z)}-\frac{1}{n}\sum_{i=1}^{n}h(z_i)}=\sup_{W\in\cW_{\rho}}\del{\cQ(W)-\hcQ(W)}\le2\Rad\del{\cH\circ S}+3\sqrt{\frac{\ln(2/\delta)}{2n}}.
    \end{align*}
    Since for each $1\le i\le n$, the mapping $z\mapsto-\ell'(y_iz)$ is
    $(1/4)$-Lipschitz, \Cref{fact:contraction} further ensures that
    $\Rad\del{\cH\circ S}\le\Rad\del{\cF_{\rho}\circ X}/4$, and thus
    \begin{align}\label{eq:gen_tmp1}
        \sup_{W\in\cW_{\rho}}\del{\cQ(W)-\hcQ(W)}\le \frac{\rho\sqrt{m}}{2\sqrt{n}}+3\sqrt{\frac{\ln(2/\delta)}{2n}}.
    \end{align}

    On the other hand, \Cref{fact:erm} ensures that under the conditions of
    \Cref{fact:gen}, for any fixed dataset, with probability $1-3\delta$ over
    the random initialization, we have
    \begin{align*}
        \hcQ(\W{k})\le\hcR(\W{k})\le\epsilon,\quad\textrm{and}\quad \enVert{\w{s}{k}-\w{s}{0}}_2\le \frac{4\lambda}{\gamma\sqrt{m}}.
    \end{align*}
    As a result, invoking \cref{eq:gen_tmp1} with
    $\rho=4\lambda/(\gamma\sqrt{m})$, with probability $1-4\delta$ over
    the random initialization and data sampling,
    \begin{align*}
        \cQ(\W{k})\le\hcQ(\W{k})+\frac{2\lambda}{\gamma\sqrt{n}}+3\sqrt{\frac{\ln(2/\delta)}{2n}}\le\epsilon+\frac{8\del{\sqrt{2\ln(4n/\delta)}+\ln(4/\epsilon)}}{\gamma^2\sqrt{n}}+3\sqrt{\frac{\ln(2/\delta)}{2n}}.
    \end{align*}
    Invoking $P_{(x,y)\sim\cD}\del{yf(x;W,a)\le0}\le2\cQ(W)$ finishes the proof.
\end{proof}

\section{Omitted proofs from \Cref{sec:sgd}}\label{app_sec:sgd}

\begin{proof}[Proof of \Cref{fact:squared_dist_sgd}]
    Recall that $\enVert{\nf_t(\W{t})}_F\le1$, we have
    \begin{align}\label{eq:step_sgd}
        \enVert{\W{t+1}-\barW}_F^2 & \le\enVert{\W{t}-\barW}_F^2-2\eta\ell'\del{y_tf_t(\W{t})}y_t\ip{\nf_t(\W{t})}{\W{t}-\barW}+\eta^2\del{\ell'\del{y_tf_t(\W{t})}}^2.
    \end{align}
    Similar to the proof of \Cref{fact:squared_dist}, the first order term of
    \cref{eq:step_sgd} can be handled using the convexity of $\ell$ and
    homogeneity of ReLU as follows
    \begin{align}\label{eq:1st_order_sgd}
        \ell'\del{y_tf_t(\W{t})}y_t\ip{\nf_t(\W{t})}{\W{t}-\barW}\ge\cR_t(\W{t})-\cR_t\del[1]{\barW},
    \end{align}
    and the second-order term of \cref{eq:step_sgd} can be bounded as follows
    \begin{align}\label{eq:2nd_order_sgd}
        \eta^2\del{\ell'\del{y_tf_t(\W{t})}}^2\le-\eta\ell'\del{y_tf_t(\W{t})}\le\eta\ell\del{y_tf_t(\W{t})}=\eta\cR_t(\W{t}),
    \end{align}
    since $\eta,-\ell'\le1$ and $-\ell'\le\ell$.
    Combining \cref{eq:step_sgd,eq:1st_order_sgd,eq:2nd_order_sgd} gives
    \begin{align*}
        \eta\cR_t(\W{t})\le\enVert{\W{t}-\barW}_F^2-\enVert{\W{t+1}-\barW}_F^2+2\eta\cR_t\del[1]{\barW}.
    \end{align*}
    Telescoping gives the claim.
\end{proof}

With \Cref{fact:squared_dist_sgd}, we give the following result, which is an
extension of \Cref{fact:erm} to the SGD setting.
\begin{lemma}\label{fact:erm_sgd}
    Under \Cref{cond:kernel_sep_gen}, given any $\epsilon\in(0,1)$, any
    $\delta\in(0,1/3)$, and any positive integer $n_0$, let
    \begin{align*}
        \lambda:=\frac{\sqrt{2\ln(4n_0/\delta)}+\ln(4/\epsilon)}{\gamma/4},\quad\textrm{and}\quad M:=\frac{4096\lambda^2}{\gamma^6}.
    \end{align*}
    For any $m\ge M$ and any constant step size $\eta\le1$, if
    $n_0\ge n:=\lceil\nicefrac{2\lambda^2}{\eta\epsilon}\rceil$, then with
    probability $1-3\delta$,
    \begin{align*}
        \frac{1}{n}\sum_{i<n}^{}\cQ_i(\W{i})\le\epsilon.
    \end{align*}
\end{lemma}
\begin{proof}
    We first sample $n_0$ data examples
    $(x_0,y_0),\ldots,(x_{n_0-1},y_{n_0-1})$, and then feed $(x_i,y_i)$ to SGD
    at step $i$.
    We only consider the first $n_0$ steps.

    The proof is similar to the proof of \Cref{fact:erm}.
    Let $n_1$ denote the first step before $n_0$ such that there exists some
    $1\le s\le m$ with
    $\enVert{\w{s}{n_1}-\w{s}{0}}_2>4\lambda/(\gamma\sqrt{m})$.
    If such a step does not exist, let $n_1=n_0$.

    Let $\barW:=\W{0}+\lambda\barU$, in exactly the same way as in
    \Cref{fact:erm}, we can show that with probability $1-3\delta$, for any
    $0\le i<n_1$,
    \begin{align*}
        y_i\ip{\nf_i(\W{i})}{\barW}\ge\ln\del{\frac{4}{\epsilon}},\quad\textrm{and thus}\quad\cR_i\del[1]{\barW}\le\epsilon/4.
    \end{align*}

    Now consider $n:=\lceil\nicefrac{2\lambda^2}{\eta\epsilon}\rceil$.
    Using \Cref{fact:squared_dist_sgd}, in the same way as the proof of
    \Cref{fact:erm} (replacing $\hcQ(\W{\tau})$ with $\cQ_i(\W{i})$, etc.), we
    can show that $n\le n_1$.
    Then invoking \Cref{fact:squared_dist_sgd} again, we get
    \begin{align*}
        \frac{1}{n}\sum_{i<n}^{}\cQ_i(\W{i})\le \frac{1}{n}\sum_{i<n}^{}\cR_i(\W{i})\le \frac{\enVert[1]{\W{0}-\barW}_F^2}{\eta n}+\frac{2}{n}\sum_{i<n}^{}\cR_i\del[1]{\barW}\le \frac{\epsilon}{2}+\frac{\epsilon}{2}=\epsilon.
    \end{align*}
\end{proof}

Next we prove \Cref{fact:gen_sgd}.
We need the following martingale Bernstein bound.
\begin{lemma}\citep[Theorem 1]{BLLRS11}\label{fact:mart_bern}
    Let $(M_t,\mathcal{F}_t)_{t\ge0}$ denote a martingale with $M_0=0$ and
    $\mathcal{F}_0$ be the trivial $\sigma$-algebra.
    Let $(\Delta_t)_{t\ge1}$ denote the corresponding martingale difference
    sequence, and let
    \begin{equation*}
        V_t:=\sum_{j=1}^{t}\mathbb{E}\left[\Delta_j^2\middle|\mathcal{F}_{j-1}\right]
    \end{equation*}
    denote the sequence of conditional variance.
    If $\Delta_t\le R$ a.s., then for any $\delta\in(0,1)$, with probability at
    least $1-\delta$,
    \begin{equation*}
        M_t\le \frac{V_t}{R}(e-2)+R\ln\left(\frac{1}{\delta}\right).
    \end{equation*}
\end{lemma}
\begin{proof}[Proof of \Cref{fact:gen_sgd}]
    For any $i\ge0$, let $z_i$ denote $(x_i,y_i)$, and $z_{0,i}$ denote
    $(z_0,\ldots,z_i)$.
    Note that the quantity $\sum_{t<i}^{}\del{\cQ(\W{t})-\cQ_t(\W{t})}$ is a
    martingale w.r.t. the filtration $\sigma(z_{0,i-1})$.
    The martingale difference sequence is given by
    $\cQ(\W{t})-\cQ_t(\W{t})$, which satisfies
    \begin{align}\label{eq:sgd_tmp1}
        \cQ(\W{t})-\cQ_t(\W{t})=\bbE_{(x,y)\sim\cD}\sbr{-\ell'\del{yf(x;\W{t},a)}}+\ell'\del{y_tf(x_t;\W{t},a)}\le1,
    \end{align}
    since $-1\le\ell'\le0$.
    Moreover, we have
    \begin{align}
    \begin{split}\label{eq:sgd_tmp2}
         & \,\bbE\sbr{\del{\cQ(\W{t})-\cQ_t(\W{t})}^2\middle|\sigma(z_{0,t-1})} \\
        = & \,\cQ(\W{t})^2-2\cQ(\W{t})\bbE\sbr{\cQ_t(\W{t})\middle|\sigma(z_{0,t-1})}+\bbE\sbr{\cQ_t(\W{t})^2\middle|\sigma(z_{0,t-1})} \\
        = & \,-\cQ(\W{t})^2+\bbE\sbr{\cQ_t(\W{t})^2\middle|\sigma(z_{0,t-1})} \\
        \le & \,\bbE\sbr{\cQ_t(\W{t})^2\middle|\sigma(z_{0,t-1})} \\
        \le & \,\bbE\sbr{\cQ_t(\W{t})\middle|\sigma(z_{0,t-1})} \\
        = & \,\cQ(\W{t}).
    \end{split}
    \end{align}
    Invoking \Cref{fact:mart_bern} with \cref{eq:sgd_tmp1,eq:sgd_tmp2} gives
    that with probability $1-\delta$,
    \begin{align*}
        \sum_{t<i}^{}\del{\cQ(\W{t})-\cQ_t(\W{t})}\le(e-2)\sum_{t<i}^{}\cQ(\W{t})+\ln\del{\frac{1}{\delta}}.
    \end{align*}
    Consequently,
    \begin{align*}
        \sum_{t<i}^{}\cQ(\W{t})\le 4\sum_{t<i}^{}\cQ_t(\W{t})+4\ln\del{\frac{1}{\delta}}.
    \end{align*}
\end{proof}

Finally, we prove \Cref{fact:sgd}.
\begin{proof}[Proof of \Cref{fact:sgd}]
    Suppose the condition of \Cref{fact:erm_sgd} holds.
    Then we have for $n=\lceil\nicefrac{2\lambda^2}{\eta\epsilon}\rceil$, with
    probability $1-3\delta$,
    \begin{align*}
        \frac{1}{n}\sum_{i<n}^{}\cQ_i(\W{i})\le\epsilon.
    \end{align*}
    Further invoking \Cref{fact:gen_sgd} gives that with probability
    $1-4\delta$,
    \begin{align*}
        \frac{1}{n}\sum_{i<n}^{}\cQ(\W{i})\le \frac{4}{n}\sum_{i<n}^{}\cQ_i(\W{i})+\frac{4}{n}\ln\del{\frac{1}{\delta}}\le 5\epsilon.
    \end{align*}
    Since $P_{(x,y)\sim\cD}\del{yf(x;W,a)\le0}\le2\cQ(W)$, we get
    \begin{align*}
        \frac{1}{n}\sum_{i=1}^{n}P_{(x,y)\sim\cD}\del{yf(x;\W{i},a)\le0}\le 10\epsilon.
    \end{align*}

    For the condition of \Cref{fact:erm_sgd} to hold, it is enough to let
    \begin{align*}
        n_0=\Theta\del{\frac{\ln(1/\delta)}{\eta\gamma^2\epsilon^2}},
    \end{align*}
    which gives
    \begin{align*}
        M=\Theta\del{\frac{\ln(1/\delta)+\ln(1/\epsilon)^2}{\gamma^8}}\quad\textrm{and}\quad n=\Theta\del{\frac{\ln(1/\delta)+\ln(1/\epsilon)^2}{\gamma^2\epsilon}}.
    \end{align*}
\end{proof}

\section{Omitted proofs from \Cref{sec:sep}}\label{app_sec:sep}

\begin{proof}[Proof of \Cref{fact:barv_kernel}]
    Define $f:\mathcal{H}\to \mathbb{R}$ by
    \begin{align*}
        f(w):=\frac{1}{2}\int\|w(z)\|_2^2\dif\mu_{\mathcal{N}}(z)=\frac{1}{2}\|w\|_{\mathcal{H}}^2.
    \end{align*}
    It holds that $f$ is continuous, and $f^*$ has the same form.
    Define $g:\mathbb{R}^n\to \mathbb{R}$ by
    \begin{align*}
        g(p):=\max_{1\le i\le n}p_i,
    \end{align*}
    with conjugate
    \begin{equation*}
        g^*(q)=
        \begin{dcases}
            0, & \textrm{if }q\in\Delta_n, \\
            +\infty, & \textrm{o.w.}
        \end{dcases}
    \end{equation*}
    Finally, define the linear mapping $A:\mathcal{H}\to \mathbb{R}^n$ by
    $(Aw)_i=y_i\ip{w}{\phi_i}_\cH$.

    Since $f$, $f^*$, $g$ and $g^*$ are lower semi-continuous, and
    $\mathbf{dom}\,g-A \mathbf{dom}\,f=\mathbb{R}^n$, and
    $\mathbf{dom}\,f^*-A^*\mathbf{dom}\,g^*=\mathcal{H}$,
    Fenchel duality may be applied in each direction \citep[Theorem 4.4.3]{borwein_var},
    and ensures that
    \begin{align*}
        \inf_{w\in \mathcal{H}}\del{f(w)+g(Aw)}=\sup_{q\in \mathbb{R}^n}\del{-f^*(A^*q)-g^*(-q)}.
    \end{align*}
    with optimal primal-dual solutions $(\barw,\barq)$.
    Moreover
    \begin{align*}
        \inf_{w\in \mathcal{H}}\left(f(w)+g(Aw)\right) & =\inf_{w\in \mathcal{H},u\in \mathbb{R}^n}\sup_{q\in \mathbb{R}^n}\left(f(w)+g(Aw+u)+\langle q,u\rangle\right) \\
         & \ge\sup_{q\in \mathbb{R}^n}\inf_{w\in \mathcal{H},u\in \mathbb{R}^n}\left(f(w)+g(Aw+u)+\langle q,u\rangle\right) \\
         & =\sup_{q\in \mathbb{R}^n}\inf_{w\in \mathcal{H},u\in \mathbb{R}^n}\del{\del{f(w)-\ip{A^*q}{w}}_\cH+\del{g(Aw+u)-\ip{-q}{Aw+u}}} \\
         & =\sup_{q\in \mathbb{R}^n}\del{-f^*(A^*q)-g^*(-q)}.
    \end{align*}
    By strong duality, the inequality holds with equality.
    It follows that
    \begin{align*}
        \barw=A^*\barq,\quad\textrm{and}\quad \mathbf{supp}(-\barq)\subset\argmax_{1\le i\le n}\,(A\barw)_i.
    \end{align*}

    Now let us look at the dual optimization problem.
    It is clear that
    \begin{align*}
        \sup_{q\in \mathbb{R}^n}\del{-f^*(A^*q)-g^*(-q)}=-\inf_{q\in\Delta_n}f^*(A^*q).
    \end{align*}
    In addition, we have
    \begin{align*}
        f^*(A^*q) & =\frac{1}{2}\int\enVert{\sum_{i=1}^{n}q_iy_i\phi_i(z)}_2^2\dif\mu_{\mathcal{N}}(z) \\
         & =\frac{1}{2}\int \sum_{i,j=1}^{n}q_iq_jy_iy_j\ip{\phi_i(z)}{\phi_j(z)}\dif\mu_{\mathcal{N}}(z) \\
         & =\frac{1}{2}\sum_{i,j=1}^{n}q_iq_jy_iy_j\int\ip{\phi_i(z)}{\phi_j(z)}\dif\mu_{\mathcal{N}}(z) \\
         &=\frac{1}{2}\sum_{i,j=1}^{n}q_iq_jy_iy_jK_1(i,j)=\frac{1}{2}(q\odot y)^\top K_1(q\odot y),
    \end{align*}
    and thus $f^*(A^*\barq)=\gamma_1^2/2$.
    Since $\barw=A^*\barq$, we have that
    $\enVert{\barw}_{\mathcal{H}}=\gamma_1$.
    In addition,
    \begin{align*}
        g(A\barw)=-f^*\del{A^*\barq}-f\del{\barw}=-\gamma_1^2,
    \end{align*}
    and thus $-\barw$ has margin $\gamma_1^2$.
    Moreover, we have
    \begin{align*}
        \barw(z)=\sum_{i=1}^{n}\barq_iy_i\phi_i(z)=\sum_{i=1}^{n}\barq_iy_ix_i\1\sbr{\ip{z}{x_i}>0},
    \end{align*}
    and thus $\enVert{\barw(z)}_2\le1$.
    Therefore, $\hatv=-\barw/\gamma_1$ satisfies all requirements of
    \Cref{fact:barv_kernel}.
\end{proof}

\begin{proof}[Proof of \Cref{fact:random_label}]
    Let $\hatq$ denote the uniform probability vector
    $(\nicefrac{1}{n},\ldots,\nicefrac{1}{n})$.
    Note that
    \begin{align*}
        \bbE_{\epsilon\sim \mathrm{unif}\del{\{-1,+1\}^n}}\sbr{\del{\hatq\odot\epsilon}^\top K_1\del{\hatq\odot\epsilon}} & =\bbE_{\epsilon\sim \mathrm{unif}\del{\{-1,+1\}^n}}\sbr{\sum_{i,j=1}^{n}\frac{1}{n^2}\epsilon_i\epsilon_jK_1(x_i,x_j)} \\
         & =\frac{1}{n^2}\sum_{i,j=1}^{n}\bbE_{\epsilon\sim \mathrm{unif}\del{\{-1,+1\}^n}}\sbr{\epsilon_i\epsilon_jK_1(x_i,x_j)} \\
         & =\frac{1}{n^2}\sum_{i=1}^{n}K_1(x_i,x_i)=\frac{1}{2n}.
    \end{align*}
    Since $0\le\del{\hatq\odot\epsilon}^\top K_1\del{\hatq\odot\epsilon}\le1$
    for any $\epsilon$, by Markov's inequality with probability $0.9$, it holds
    that $\del{\hatq\odot\epsilon}^\top K_1\del{\hatq\odot\epsilon}\le1/(20n)$,
    and thus $\gamma_1\le1/\sqrt{20n}$.
\end{proof}

\begin{proof}[Proof of \Cref{fact:2_xor_gamma}]
    By symmetry, we only need to consider an $(x,y)$ where
    $(x_1,x_2,y)=(\nicefrac{1}{\sqrt{d-1}},0,1)$.
    Let $z_{p,q}$ denote $(z_p,z_{p+1},\ldots,z_q)$, and similarly define
    $x_{p,q}$.
    We have
    \begin{align}
         & \ y\int\ip{\barv(z)}{x}\1\sbr{\ip{z}{x}>0}\dif\mu_{\cN}(z) \nonumber \\
        = & \ y\int\del{\int\ip{\barv(z)}{x}\1\sbr{\ip{z}{x}>0}\dif\mu_{\cN}(z_{3,d})}\dif\mu_{\cN}(z_{1,2}) \label{eq:2_xor_tmp1} \\
        = & \ y\int\ip{\barv(z)_{1,2}}{x_{1,2}}\del{\int\1\sbr{\ip{z_{1,2}}{x_{1,2}}+\ip{z_{3,d}}{x_{3,d}}>0}\dif\mu_{\cN}(z_{3,d})}\dif\mu_{\cN}(z_{1,2}) \label{eq:2_xor_tmp2} \\
        = & \ \sum_{i=1}^{4}y\int\ip{\barv(z)_{1,2}}{x_{1,2}}\del{\int\1\sbr{\ip{z_{1,2}}{x_{1,2}}+\ip{z_{3,d}}{x_{3,d}}>0}\dif\mu_{\cN}(z_{3,d})}\1\sbr{z_{1,2}\in A_i}\dif\mu_{\cN}(z_{1,2}), \label{eq:four_regions}
    \end{align}
    where \cref{eq:2_xor_tmp1} is due to the independence between $z_{1,2}$ and
    $z_{3,d}$, and in \cref{eq:2_xor_tmp2} we use the fact that $\barv(z)_{1,2}$
    only depends on $z_{1,2}$ and $\barv(z)_{3,d}$ are all zero.
    Since $\ip{\barv(z)_{1,2}}{x_{1,2}}=0$ for $z_{1,2}\in A_2\cup A_4$, we only
    need to consider $A_1$ and $A_3$ in \cref{eq:four_regions}.
    For simplicity, we will denote $z_{1,2}$ by $p\in\R^2$, and $\barv(z)_{1,2}$
    by $\barv(p)$, and $z_{3,d}$ by $q\in\R^{d-2}$.

    For any nonzero $p\in A_1$, we have $-p\in A_3$, and
    $\ip{\barv(p)}{x_{1,2}}=1/\sqrt{d-1}$.
    Therefore
    \begin{align}
         & \ y\ip{\barv(p)}{x_{1,2}}\del{\int\1\sbr{\ip{p}{x_{1,2}}+\ip{q}{x_{3,d}}>0}\dif\mu_{\cN}(q)} \nonumber \\
         & \ +y\ip{\barv(-p)}{x_{1,2}}\del{\int\1\sbr{\ip{-p}{x_{1,2}}+\ip{q}{x_{3,d}}>0}\dif\mu_{\cN}(q)} \nonumber \\
        = & \ \frac{1}{\sqrt{d-1}}\int\del{\1\sbr{\frac{p_1}{\sqrt{d-1}}+\ip{q}{x_{3,d}}>0}-\1\sbr{\frac{-p_1}{\sqrt{d-1}}+\ip{q}{x_{3,d}}>0}}\dif\mu_{\cN}(q) \nonumber \\
        = & \ \frac{1}{\sqrt{d-1}}\bbP\del{\frac{-p_1}{\sqrt{d-1}}\le\ip{q}{x_{3,d}}\le \frac{p_1}{\sqrt{d-1}}}. \label{eq:2_xor_tmp3}
    \end{align}
    Let $\varphi$ denote the density function of the standard Gaussian
    distribution, and for $c>0$, let $U(c)$ denote the probability that a
    standard Gaussian random variable lies in the interval $[-c,c]$:
    \begin{align*}
        U(c):=\int_{-c}^c\varphi(t)\dif t.
    \end{align*}
    Since $\ip{q}{x_{3,d}}$ is a Gaussian variable with standard deviation
    $\sqrt{\nicefrac{(d-2)}{(d-1)}}$, we have
    \begin{align}\label{eq:2_xor_tmp4}
        \bbP\del{\frac{-p_1}{\sqrt{d-1}}\le\ip{q}{x_{3,d}}\le \frac{p_1}{\sqrt{d-1}}}=U\del{\frac{p_1}{\sqrt{d-2}}}.
    \end{align}
    Plugging \cref{eq:2_xor_tmp3,eq:2_xor_tmp4} into \cref{eq:four_regions}
    gives:
    \begin{align*}
        y\int\ip{\barv(z)}{x}\1\sbr{\ip{z}{x}>0}\dif\mu_{\cN}(z) & =\frac{1}{\sqrt{d-1}}\int U\del{\frac{p_1}{\sqrt{d-2}}}\1\sbr{p\in A_1}\dif\mu_{\cN}(p) \\
         & =\frac{1}{\sqrt{d-1}}\int_0^\infty U\del{\frac{p_1}{\sqrt{d-2}}}\del{\int_{-p_1}^{p_1}\varphi(p_2)\dif p_2}\varphi(p_1)\dif p_1 \\
         & =\frac{1}{\sqrt{d-1}}\int_0^\infty U\del{\frac{p_1}{\sqrt{d-2}}}U(p_1)\varphi(p_1)\dif p_1 \\
         & \ge \frac{1}{\sqrt{d-1}}\int_0^1U\del{\frac{p_1}{\sqrt{d-2}}}U(p_1)\varphi(p_1)\dif p_1.
    \end{align*}
    For $t\in[-1,+1]$, it holds that $\varphi(t)\ge1\sqrt{2\pi e}$, and thus
    \begin{align*}
        U(a)=\int_{-a}^a\varphi(t)\dif t\ge \frac{2a}{\sqrt{2\pi e}}.
    \end{align*}
    Therefore \cref{eq:four_regions} is lower bounded by
    \begin{align*}
        \frac{1}{\sqrt{d-1}}\int_0^1U\del{\frac{p_1}{\sqrt{d-2}}}U(p_1)\varphi(p_1)\dif p_1 & \ge \frac{1}{\sqrt{d-1}}\int_0^1 \frac{2}{\sqrt{2\pi e}}\cdot \frac{p_1}{\sqrt{d-2}}\cdot \frac{2p_1}{\sqrt{2\pi e}}\cdot\frac{1}{\sqrt{2\pi e}}\dif p_1 \\
         & \ge\frac{1}{20\sqrt{(d-1)(d-2)}}\int_0^1p_1^2\dif p_1 \\
         & =\frac{1}{60\sqrt{(d-1)(d-2)}} \\
         & \ge \frac{1}{60d}.
    \end{align*}
\end{proof}

To prove \Cref{fact:ntk_lb}, we need the following technical lemma.
\begin{lemma}\label{fact:normal_prop}
    Given $z_1\sim\cN(0,1)$ and $z_2\sim\cN(0,b^2)$ that are independent where
    $b>1$, we have
    \begin{align*}
        \bbP\del{|z_1|<|z_2|}>1-\frac{1}{b}.
    \end{align*}
\end{lemma}
\begin{proof}
    First note that for $z_3\sim\cN(0,1)$ which is independent of $z_1$,
    \begin{align*}
        \bbP\del{|z_1|<|z_2|}=\bbP\del{|z_1|<b|z_3|}=1-\bbP\del{|z_3|<\frac{1}{b}|z_1|}.
    \end{align*}
    Still let $\varphi$ denote the density of $\cN(0,1)$, and let $U(c)$ denote
    the probability that $z_3\in[-c,c]$.
    We have
    \begin{align*}
        \bbP\del{|z_3|<\frac{1}{b}|z_1|} & =\int\int\1\sbr{|z_3|<\frac{1}{b}|z_1|}\varphi(z_3)\varphi(z_1)\dif z_3\dif z_1 \\
         & =\int U\del{\frac{1}{b}|z_1|}\varphi(z_1)\dif z_1 \\
         & \le \frac{2}{\sqrt{2\pi}b}\int|z_1|\varphi(z_1)\dif z_1= \frac{2}{\pi b}<\frac{1}{b},
    \end{align*}
    where we use the facts that $U(c)\le2c/\sqrt{2\pi}$ and
    $\bbE[|z_1|]=\sqrt{2/\pi}$.
\end{proof}

We now give the proof of \Cref{fact:ntk_lb} using \Cref{fact:normal_prop}.
\begin{proof}[Proof of \Cref{fact:ntk_lb}]
    By symmetry, we only need to consider the following training set:
    \begin{align*}
        x_1=(1,0,1,\ldots,1), & \quad y_1=1, \\
        x_2=(0,1,1,\ldots,1), & \quad y_2=-1, \\
        x_3=(-1,0,1,\ldots,1), & \quad y_3=1, \\
        x_4=(0,-1,1,\ldots,1), & \quad y_4=-1.
    \end{align*}
    The $1/\sqrt{d-1}$ factor is omitted also because we only discuss the $0/1$
    loss.

    For any $s$, let $A_s$ denote the event that
    \begin{align*}
        \1\sbr{\ip{w_s}{x_1}>0}=\1\sbr{\ip{w_s}{x_2}>0}=\1\sbr{\ip{w_s}{x_3}>0}=\1\sbr{\ip{w_s}{x_4}>0}.
    \end{align*}
    We will show that if $m\le\sqrt{d-2}/4$, then $A_s$ is true for all
    $1\le s\le m$ with probability $1/2$, and \Cref{fact:ntk_lb} follows from
    the fact that the XOR data is not linearly separable.

    For any $s$ and $i$,
    \begin{align*}
        \ip{w_s}{x_i}=(w_s)_1(x_i)_1+(w_s)_2(x_i)_2+\sum_{j=3}^{d}(w_s)_j.
    \end{align*}
    Since $\del{(x_i)_1,(x_i)_2}$ is $(1,0)$ or $(0,1)$ or $(-1,0)$ or $(0,-1)$,
    event $A_s$ will happen as long as
    \begin{align*}
        \envert{(w_s)_1}<\envert{\sum_{j=3}^{d}(w_s)_j},\quad\textrm{and}\quad\envert{(w_s)_2}<\envert{\sum_{s=3}^{d}(w_s)_j}.
    \end{align*}
    Note that $(w_s)_1,(w_s)_2\sim\cN(0,1)$ while
    $\sum_{j=3}^{d}(w_s)_j\sim\cN(0,d-2)$.
    As a result, due to \Cref{fact:normal_prop},
    \begin{align*}
        \bbP\del{\envert{(w_s)_1}<\envert{\sum_{j=3}^{d}(w_s)_j}}=\bbP\del{\envert{(w_s)_2}<\envert{\sum_{s=3}^{d}(w_s)_j}}>1-\frac{1}{\sqrt{d-2}}.
    \end{align*}
    Using a union bound, $\bbP(A_s)>1-\nicefrac{2}{\sqrt{d-2}}$.
    If $m\le\sqrt{d-2}/4$, then by a union bound again,
    \begin{align*}
        \bbP\del{\bigcup_{1\le s\le m}A_s}>1-\frac{2}{\sqrt{d-2}}m\ge1-\frac{2}{\sqrt{d-2}}\frac{\sqrt{d-2}}{4}=\frac{1}{2}.
    \end{align*}
\end{proof}

\end{document}